\documentclass[sigconf]{acmart}
\AtBeginDocument{%
  \providecommand\BibTeX{{%
    \normalfont B\kern-0.5em{\scshape i\kern-0.25em b}\kern-0.8em\TeX}}}

\setcopyright{acmlicensed}
\copyrightyear{2018}
\acmYear{2018}
\acmDOI{XXXXXXX.XXXXXXX}

\acmConference[Conference acronym 'XX]{Make sure to enter the correct
  conference title from your rights confirmation emai}{June 03--05,
  2018}{Woodstock, NY}
\acmISBN{978-1-4503-XXXX-X/18/06}

\usepackage{booktabs}
\usepackage{algorithm}
\usepackage{multirow}
\usepackage{pifont}
\usepackage{xcolor}
\usepackage{adjustbox}
\usepackage{makecell}
\usepackage{enumitem}
\usepackage{algpseudocode}
\usepackage{mathtools}
\usepackage{tcolorbox}
\usepackage{colortbl}
\usepackage{subcaption}

\theoremstyle{definition}
\newtheorem{definition}{Definition}
\newtheorem{theorem}{Theorem}

\newtheorem{problem}{Problem}

\usepackage{pifont}

\usepackage[framemethod=tikz]{mdframed}
\newcounter{finding}
\numberwithin{finding}{section}

\title{Efficient Cross-Architecture Knowledge Transfer for Large-Scale Online User Response Prediction}

\author{Yucheng Wu}
\authornote{This work was done when the author was an intern at Tencent.}
\authornote{Both authors contributed equally to this research.}
\affiliation{
  \institution{Key Lab of High Confidence Software Technologies (Peking University), Ministry of Education \& School of Computer Science, Peking University}
  \city{Beijing}
  \country{China}
}
\email{wuyucheng@stu.pku.edu.cn}

\author{Yuekui Yang}
\authornotemark[2]
\authornote{Corresponding authors.}
\affiliation{
  \institution{Department of Computer Science and Technology, Tsinghua University \& Advertising Engineering Department, CDG, Tencent Corporation}
  \city{Beijing}
  \country{China}
}
\email{yuekuiyang@tencent.com}

\author{Hongzheng Li}
\affiliation{
  \institution{Advertising Engineering Department, CDG, Tencent Corporation}
  \city{Beijing}
  \country{China}
}
\email{ethanleeli@tencent.com}

\author{Anan Liu}
\affiliation{
  \institution{Advertising Engineering Department, CDG, Tencent Corporation}
  \city{Beijing}
  \country{China}
}
\email{ananliu@tencent.com}

\author{Jian Xiao}
\affiliation{
  \institution{Advertising Engineering Department, CDG, Tencent Corporation}
  \city{Beijing}
  \country{China}
}
\email{jakexiao@tencent.com}

\author{Junjie Zhai}
\affiliation{
  \institution{Advertising Engineering Department, CDG, Tencent Corporation}
  \city{Beijing}
  \country{China}
}
\email{jasonzhai@tencent.com}

\author{Huan Yu}
\affiliation{
  \institution{Advertising Engineering Department, CDG, Tencent Corporation}
  \city{Beijing}
  \country{China}
}
\email{huanyu@tencent.com}

\author{Shaoping Ma}
\affiliation{
  \institution{Department of Computer Science and Technology, Tsinghua University}
  \city{Beijing}
  \country{China}
}
\email{msp@tsinghua.edu.cn}

\author{Leye Wang}
\authornotemark[3]
\affiliation{
  \institution{Key Lab of High Confidence Software Technologies (Peking University), Ministry of Education \& School of Computer Science, Peking University}
  \city{Beijing}
  \country{China}
}
\email{leyewang@pku.edu.cn}

\begin{abstract}
Deploying new architectures in large-scale user response prediction systems incurs high model switching costs due to expensive retraining on massive historical data and performance degradation under data retention constraints. Existing knowledge distillation methods struggle with architectural heterogeneity and the prohibitive cost of transferring large embedding tables.
We propose CrossAdapt, a two-stage framework for efficient cross-architecture knowledge transfer. The offline stage enables rapid embedding transfer via dimension-adaptive projections without iterative training, combined with progressive network distillation and strategic sampling to reduce computational cost. The online stage introduces asymmetric co-distillation, where students update frequently while teachers update infrequently, together with a distribution-aware adaptation mechanism that dynamically balances historical knowledge preservation and fast adaptation to evolving data.
Experiments on three public datasets show that CrossAdapt achieves 0.27–0.43\% AUC improvements while reducing training time by 43–71\%. Large-scale deployment on Tencent WeChat Channels ($\sim$10M daily samples) further demonstrates its effectiveness, significantly mitigating AUC degradation, LogLoss increase, and prediction bias compared to standard distillation baselines.
\end{abstract}



\begin{document}
\maketitle

\section{Introduction}

% Click-through rate (CTR) prediction is a cornerstone of modern recommender systems and online advertising platforms, where even slight improvements in accuracy can yield significant revenue gains \cite{zhou2018deep, zhang2021deep}. In large-scale industrial applications, CTR models operate within highly dynamic environments driven by continuous data streams and evolving user behaviors. To maintain responsiveness, they often rely on \textit{online learning} algorithms that incrementally update model parameters based on real-time feedback, enabling rapid adaptation to shifting trends and patterns \cite{mcmahan2013ad, he2014practical}. The rapid evolution of deep learning has catalyzed frequent architectural innovations in CTR prediction models, progressing from simple logistic regression to complex deep architectures such as Wide\&Deep \cite{cheng2016wide}, DeepFM \cite{guo2017deepfm}, and transformer-based models \cite{song2019autoint}. 

User response prediction, including click-through rate (CTR) and conversion rate (CVR) prediction tasks, is a cornerstone of modern recommender systems and online advertising platforms, where even slight improvements in accuracy can yield significant revenue gains \cite{zhou2018deep, zhang2021deep}. In large-scale industrial applications, user response prediction models operate within highly dynamic environments driven by continuous data streams and evolving user behaviors. To maintain responsiveness, they often rely on online learning algorithms that incrementally update model parameters based on real-time feedback, enabling rapid adaptation to shifting trends and patterns \cite{mcmahan2013ad, he2014practical}. The rapid evolution of deep learning has catalyzed frequent architectural innovations in user response prediction models, progressing from simple logistic regression to complex deep architectures such as Wide\&Deep \cite{cheng2016wide}, DeepFM \cite{guo2017deepfm}, and transformer-based models \cite{song2019autoint}. These architectures serve as foundations for pCTR/pCVR models (CTR/CVR predicted models), which estimate the probability of user clicks/conversions.

% While these architectural advances promise superior predictive performance and enhanced feature interaction modeling capabilities, deploying new model architectures in production systems presents significant challenges that we term as \textbf{model switching cost}. 
% The model switching cost manifests in two critical dimensions. First, \textit{computational overhead}: retraining a new architecture from scratch on historical data spanning potentially trillions of samples incurs prohibitive time and computational costs. In industrial settings where models are trained on months or years of accumulated user interaction data, complete retraining can require weeks of computation on distributed clusters, delaying the deployment of improved architectures \cite{wang2020practical}. Second, \textit{performance degradation}: due to data retention policies and storage limitations, industrial systems often cannot preserve complete historical training data. Consequently, new models trained only on available recent data suffer from insufficient exposure to long-tail patterns and rare events, resulting in suboptimal performance compared to existing models that have accumulated knowledge over extended periods \cite{zhang2020retrain}.

While these architectural advances promise superior predictive performance, deploying new model architectures in production systems presents significant challenges that we term \textbf{model switching cost}. This cost manifests in two critical dimensions. First, \textit{computational overhead}: retraining from scratch on potentially trillions of historical samples incurs prohibitive time and computational costs. In industrial settings with months or years of accumulated data, complete retraining can require weeks of distributed computation, delaying improved architecture deployment \cite{wang2020practical}. Second, \textit{performance degradation}: due to data retention policies and storage limitations, industrial systems often cannot preserve complete historical data. New models trained only on recent data suffer from insufficient exposure to long-tail patterns and rare events, resulting in suboptimal performance compared to existing models that have accumulated knowledge over extended periods \cite{zhang2020retrain}.

Knowledge distillation \cite{hinton2015distilling} presents a promising approach to mitigate model switching costs by transferring knowledge from existing models (teacher) to new architectures (student). However, applying traditional distillation to pCTR/pCVR models poses unique difficulties. First, architectural heterogeneity between teacher and student models renders intermediate layer distillation infeasible due to dimensional and semantic misalignment \cite{ba2014deep}, limiting knowledge transfer to weaker output-level signals, which provide weaker supervision compared to layer-wise alignment. Second, embedding tables encoding categorical features constitute the majority of pCTR/pCVR model parameters (often exceeding 99\%) \cite{wang2017deep,lai2023adaembed,liu2021learnable}, yet their massive cardinality makes iterative training computationally prohibitive. Third, online data streams exhibit continuous distribution shifts as user preferences evolve \cite{mcmahan2013ad}, raising questions about when historical knowledge remains beneficial versus when rapid adaptation is critical. We formulate these into two challenges:
\begin{itemize}[leftmargin=*,noitemsep]
    \item \textbf{Challenge 1: Cross-Architecture Knowledge Inheritance.} How can we efficiently transfer knowledge from a well-trained teacher to a heterogeneous student with different embedding dimensions and feature interaction mechanisms, without retraining on full historical data?
    \item \textbf{Challenge 2: Distribution-Aware Online Adaptation.} How can we design an online learning mechanism that balances the transfer of invariant historical knowledge with rapid adaptation to newly emerging dynamic information?
    % How can we design an online learning mechanism that balances historical preservation with fast adaptation, enabling the student to evolve beyond its teacher?
\end{itemize}

% Addressing both challenges simultaneously is inherently difficult. Efficient knowledge transfer relies on stable teacher guidance and historical coverage, whereas adaptive online learning demands agility and responsiveness to streaming data. A unified process would either overburden training with expensive teacher updates or constrain adaptation with outdated supervision. To resolve this tension, we separate model transfer into two coordinated phases. The \textit{offline phase} focuses on fast inheritance when deploying a new architecture: using historical data and a frozen teacher, the student inherits comprehensive knowledge efficiently. The subsequent \textit{online phase} emphasizes continual adaptation: as data distributions evolve and static supervision becomes outdated, the teacher–student pair co-evolves, enabling the student to continually refine its representation capacity.
% This phase separation is essential for maximizing transfer efficiency offline and achieving a stability–adaptability balance online. 
% Building on this insight, we propose \textbf{CrossAdapt}, a two-stage framework that integrates offline cross-architecture transfer with online adaptive co-distillation.

Addressing both challenges simultaneously is inherently difficult. Efficient knowledge transfer relies on stable teacher guidance, whereas adaptive online learning demands responsiveness to streaming data. A unified process would either overburden training with expensive teacher updates or constrain adaptation with outdated supervision. To resolve this tension, we separate model transfer into two coordinated phases. The \textit{offline phase} focuses on fast inheritance when deploying new architectures: using historical data and a frozen teacher, the student efficiently inherits comprehensive knowledge. The subsequent \textit{online phase} emphasizes continual adaptation: as distributions evolve and static supervision becomes outdated, the teacher--student pair co-evolves, enabling the student to continually refine its capacity. This phase separation is essential for maximizing transfer efficiency offline and achieving stability--adaptability balance online. 

% Building on this insight, we propose \textbf{CrossAdapt}, a two-stage framework integrating offline cross-architecture transfer with online adaptive co-distillation.
% In the \textit{offline stage} (addressing Challenge~1), knowledge is transferred through dimension-adaptive embedding projection and progressive network distillation. Embedding tables are transformed via direct mathematical mappings that optimally preserve feature relationships, while interaction networks are trained progressively–first with frozen embeddings to avoid gradient interference, then jointly optimized after stabilization. Strategic sampling reduces cost by oversampling informative positives and ensuring temporal coverage.
% In the \textit{online stage} (addressing Challenge~2), asymmetric teacher–student co-evolution enables continuous adaptation. The student updates rapidly to capture new trends, while the teacher evolves more slowly to maintain stable supervision. An adaptive mechanism monitors feature distribution shifts: during stable periods, historical samples enrich distillation; during shifts, training focuses exclusively on streaming data for rapid adaptation.

Building on this insight, we propose \textbf{CrossAdapt}, a two-stage framework integrating offline cross-architecture transfer with online adaptive co-distillation.
In the \textit{offline stage} (addressing Challenge~1), knowledge transfers through dimension-adaptive embedding projection and progressive network distillation. Embedding tables are transformed via mathematical mappings that preserve feature relationships, while interaction networks train progressively--first with frozen embeddings to avoid gradient interference, then jointly optimized. Strategic sampling reduces training cost while preserving performance by improving sample information density and diversity.
In the \textit{online stage} (addressing Challenge~2), asymmetric teacher--student co-evolution enables continuous adaptation. The student updates rapidly to capture new trends, while the teacher evolves slowly to maintain stable supervision. An adaptive mechanism monitors distribution shifts: during stability, historical samples enrich distillation; during shifts, training focuses on streaming data for rapid adaptation.

% Our key contributions are summarized as follows:
% \begin{enumerate}[leftmargin=*,noitemsep]
% \item We define the model switching cost in large-scale user response prediction systems, decomposing it into performance degradation and deployment overhead as a benchmark for evaluating model transitions.

% \item We propose CrossAdapt, a two-stage heterogeneous knowledge transfer framework that integrates offline cross-architecture transfer with online adaptive co-distillation, jointly optimizing computational efficiency and adaptability to distribution shifts.

% \item We design a dimension-adaptive projection strategy for embedding tables that enables rapid and near-lossless knowledge transfer with theoretical guarantees.

% \item Extensive experiments on public and industrial datasets show that CrossAdapt greatly accelerates model knowledge transfer and consistently surpasses standard distillation baselines.
% \end{enumerate}

Our key contributions are summarized as follows:
\begin{enumerate}[leftmargin=*,noitemsep]
\item We systematically formulate the model switching cost problem in user response prediction systems, decomposing it into performance degradation and computational overhead.
\item We propose CrossAdapt, a two-stage knowledge transfer framework that integrates offline cross-architecture transfer with online adaptive co-distillation, jointly optimizing efficiency, performance and adaptability to distribution shifts.
\item We design a dimension-adaptive projection strategy for embedding tables that enables rapid, near-lossless knowledge transfer with theoretical guarantees.
\item Extensive experiments demonstrate that CrossAdapt significantly accelerates knowledge transfer while consistently outperforming standard distillation baselines across public benchmarks and industrial deployments.\footnote{The code is available at \href{https://github.com/wuyucheng2002/CrossAdapt}{https://github.com/wuyucheng2002/CrossAdapt}.}
\end{enumerate}

% Extensive experiments on public and industrial datasets show CrossAdapt significantly accelerates knowledge transfer and consistently outperforms standard distillation baselines.

% Extensive experiments demonstrate CrossAdapt significantly accelerates knowledge transfer and consistently outperforms standard distillation baselines across public benchmarks and industrial deployment.

\begin{figure}[t]
\centering
\includegraphics[width=0.9\linewidth]{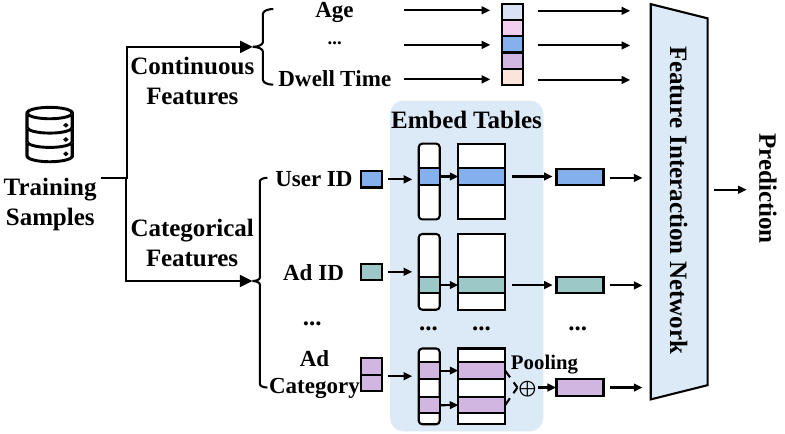}
\vspace{-1em}
\caption{A typical user response prediction model architecture, composed of embedding tables and feature interaction networks.}
\label{fig:CTR_model}
\vspace{-1em}
\end{figure}

\section{Problem Formulation}

We study efficient knowledge transfer in user response prediction. Given a feature vector $\mathbf{x} \in \mathcal{X}$ with categorical features $\mathbf{x}^{\text{cat}}$ (e.g., user/item identifiers) and numerical features $\mathbf{x}^{\text{num}}$ (e.g., behavioral statistics), a prediction model $f(\cdot; \theta)$ estimates the click/conversion probability $p(y=1|\mathbf{x}; \theta)$, where $y \in \{0, 1\}$ denotes the binary label.

\begin{definition}[Model Architecture]
\label{def:ctr_model}
A user response prediction model $f(\cdot; \theta)$ comprises two principal components. The \textbf{embedding tables} $E \in \mathbb{R}^{V \times d}$ transform sparse categorical features into dense vectors, where $V$ denotes vocabulary size (typically ranging from millions to billions) and $d$ is embedding dimension. These tables typically constitute the majority of model parameters~\cite{naumov2019deep}. The \textbf{interaction networks} $\theta^{\text{net}}$ capture feature dependencies through architectures including multilayer perceptrons, factorization machines, attention mechanisms, and transformers. The complete parameter set is $\theta = \{E, \theta^{\text{net}}\}$, as illustrated in Figure~\ref{fig:CTR_model}.
\end{definition}

This hybrid structure introduces a fundamental challenge when deploying new architectures: differences in embedding dimensions or interaction designs preclude direct parameter reuse. We formalize this challenge as the \textit{model switching cost}, which captures both computational overhead and performance impact.

% \begin{definition}[Model Switching Cost]
% \label{def:switching_cost}
% The total cost $\mathcal{C}_{\text{switch}}$ of transitioning from a teacher model $f_T(\cdot; \theta_T)$ to a student model $f_S(\cdot; \theta_S)$ decomposes into computational overhead and performance degradation:
% \begin{equation}
% \mathcal{C}_{\text{switch}} = \mathcal{C}_{\text{comp}} + \mathcal{C}_{\text{perf}}
% \end{equation}
% The \textit{computational overhead} $\mathcal{C}_{\text{comp}} = C_{\text{iter}} \cdot N$ measures the resources/time required to train and deploy the new model, where $C_{\text{iter}}$ denotes the per-sample cost and $N = |\mathcal{D}|$ is the dataset size. In industrial systems with trillions of training samples, full retraining may require months of distributed computation. The \textit{performance degradation} quantifies the loss in predictive quality:
% \begin{equation}
% \mathcal{C}_{\text{perf}} = 
% \mathbb{E}_{(\mathbf{x}, y) \sim \mathcal{D}^{\text{test}}} 
% \left[
% \mathcal{L}(f_S(\mathbf{x}; \theta_S), y)
% -
% \mathcal{L}(f_T(\mathbf{x}; \theta_T), y)
% \right]
% \end{equation}
% \end{definition}

\begin{definition}[Model Switching Cost]
\label{def:switching_cost}
The total cost $\mathcal{C}_{\text{switch}}$ of transitioning from teacher model $f_T(\cdot; \theta_T)$ to student model $f_S(\cdot; \theta_S)$ decomposes into two parts:
\begin{equation}
\mathcal{C}_{\text{switch}} = \mathcal{C}_{\text{comp}} + \mathcal{C}_{\text{perf}}
\end{equation}
The \textit{computational overhead} $\mathcal{C}_{\text{comp}} = C_{\text{iter}} \cdot N$ measures resources/time required to train the new model, where $C_{\text{iter}}$ denotes per-sample cost and $N = |\mathcal{D}|$ is dataset size. In industrial systems, full retraining may require months of distributed computation. The \textit{performance degradation} quantifies predictive quality loss:
\begin{equation}
\mathcal{C}_{\text{perf}} = 
\mathbb{E}_{(\mathbf{x}, y) \sim \mathcal{D}^{\text{test}}} 
\left[
\mathcal{L}(f_S(\mathbf{x}; \theta_S), y)
-
\mathcal{L}(f_T(\mathbf{x}; \theta_T), y)
\right]
\end{equation}
\end{definition}

\textbf{Data Specification.}
We organize data into four temporally ordered subsets. The \textit{historical data} $\mathcal{D}^{\text{hist}}$ represents earliest records no longer accessible due to storage or privacy constraints. The \textit{training data} $\mathcal{D}^{\text{train}}$ consists of available offline samples collected after $\mathcal{D}^{\text{hist}}$ for training the student. The \textit{online data} $\mathcal{D}^{\text{online}}$ arrives sequentially in real time for continual adaptation under distribution shift. The \textit{test data} $\mathcal{D}^{\text{test}}$ is drawn from the most recent period for evaluation. Crucially, the teacher $f_T$ has learned from both $\mathcal{D}^{\text{hist}}$ and $\mathcal{D}^{\text{train}}$, whereas the student $f_S$ cannot access $\mathcal{D}^{\text{hist}}$, creating a knowledge gap our framework aims to bridge.

\begin{problem}[Efficient Model Knowledge Transfer]
\label{prob:transition}
Given a teacher model $f_T(\cdot; \theta_T)$ trained on $\mathcal{D}^{\text{hist}} \cup \mathcal{D}^{\text{train}}$, and a student architecture $f_S(\cdot; \theta_S)$ with potentially different dimensions and designs, the objective is to design a transfer framework $\phi$ that minimizes the switching cost $\mathcal{C}_{\text{switch}}(\phi)$ using only $\mathcal{D}^{\text{train}}$ and streaming $\mathcal{D}^{\text{online}}$.
\end{problem}

\begin{figure}[t]
\centering
\includegraphics[width=\linewidth]{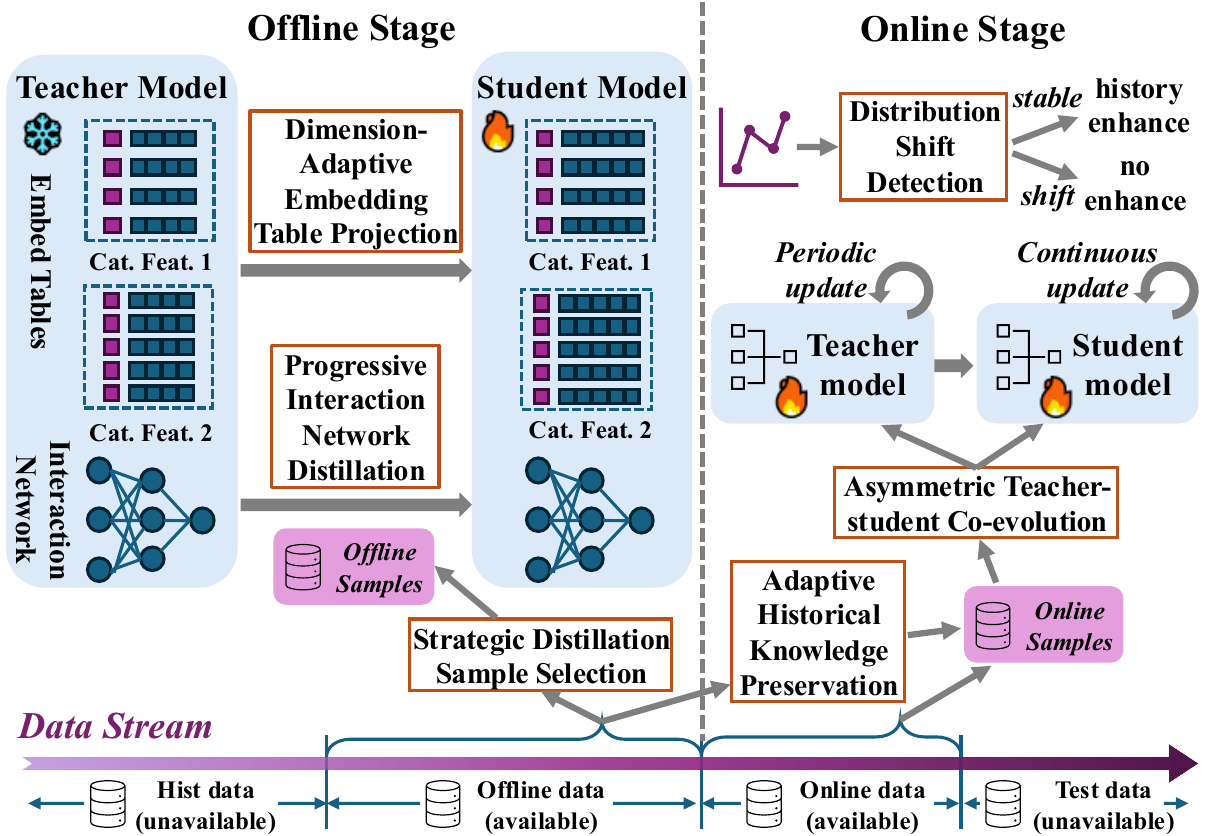}
\vspace{-1em}
\caption{Overview of CrossAdapt with offline cross-architecture transfer and online adaptive co-distillation.}
\label{fig:framework}
\vspace{-1em}
\end{figure}

\section{CrossAdapt Framework}

The proposed CrossAdapt framework adopts a two-stage knowledge transfer paradigm that effectively reduces model switching costs through efficient distillation and adaptive online learning, as illustrated in Figure~\ref{fig:framework}.

\subsection{Offline Cross-Architecture Transfer}

he offline stage addresses \textbf{Challenge 1} through efficient knowledge transfer despite architectural heterogeneity, leveraging three complementary techniques: dimension-adaptive embedding projection, progressive interaction network distillation, and strategic sampling.

\subsubsection{Dimension-Adaptive Embedding Table Projection}
\label{sec:embedding_projection}
% Embedding tables encode semantic representations of categorical features (e.g., user IDs, item IDs, categories) and typically account for the majority of model parameters in pCTR/pCVR models \cite{wang2017deep,lai2023adaembed,liu2021learnable}. The core value of these embeddings lies in the \textit{relationships} they capture: similar items should have similar embeddings, complementary items should exhibit specific angular relationships, and user-item affinities should be preserved through inner products \cite{mikolov2013distributed}.
% Traditional distillation requires iterative training over the entire parameter space, incurring substantial computational costs while potentially failing to preserve critical relational structures. We propose a parameter-free projection method that guarantees optimal preservation through mathematical transformations.

Embedding tables encode semantic representations of categorical features (e.g., user/item IDs, categories) and typically constitute the majority of pCTR/pCVR model parameters \cite{wang2017deep,lai2023adaembed,liu2021learnable}. The core value of these embeddings lies in the \textit{relationships} they capture: similar items should have similar embeddings, complementary items should exhibit specific angular relationships, and user-item affinities should be preserved through inner products \cite{mikolov2013distributed}.
Traditional distillation requires iterative training over the entire parameter space, incurring substantial costs while potentially failing to preserve relational structures. We propose a parameter-free projection method guaranteeing optimal preservation through mathematical transformations.

% Let $E_T \in \mathbb{R}^{V \times d_T}$ and $E_S \in \mathbb{R}^{V \times d_S}$ denote the teacher and student embedding matrices, where $d_T$, $d_S$ are embedding dimensions. We design optimal mappings for three scenarios: 
% \textit{Case 1: Equal dimensions} ($d_S = d_T$). When architectures share the same embedding dimension, we directly copy: $E_S = E_T$. 
% \textit{Case 2: Dimension expansion} ($d_S > d_T$). When upgrading to a model with higher capacity, we construct an orthogonal projection matrix through QR decomposition \cite{gander1980algorithms}:
% \begin{equation}
% \begin{aligned}
% R &\sim \mathcal{N}(0, 1)^{d_S \times d_S} \\
% Q&, \_ = \text{QR}(R) \quad \text{where } Q \in \mathbb{R}^{d_S \times d_S} \text{ is orthogonal}\\
% W &= Q[:, 1:d_T]^T \in \mathbb{R}^{d_T \times d_S}, \quad E_S = E_T \cdot W
% \end{aligned}
% \end{equation} 
% \textit{Case 3: Dimension reduction} ($d_S < d_T$). When computational constraints require a smaller model, we apply Principal Component Analysis (PCA) to maximally preserve variance \cite{abdi2010principal}:
% \begin{equation}
% \begin{aligned}
% \mu &= \frac{1}{V}\sum_{i=1}^V E_T[i,:]^\top, \quad \bar{E}_T = E_T - \mathbf{1}_V \mu^\top \\
% C &= \frac{1}{V}\bar{E}_T^\top \bar{E}_T \in \mathbb{R}^{d_T \times d_T}, \quad U, \Lambda, U^\top = \text{EigenDecomp}(C) \\
% W &= U[:, 1:d_S] \in \mathbb{R}^{d_T \times d_S}, \quad E_S = \bar{E}_T \cdot W + \mathbf{1}_V \mu^\top
% \end{aligned}
% \end{equation}

Let $E_T \in \mathbb{R}^{V \times d_T}$ and $E_S \in \mathbb{R}^{V \times d_S}$ denote teacher and student embedding matrices with dimensions $d_T$, $d_S$. We design optimal mappings for three scenarios: 
\textit{Case 1: Equal dimensions} ($d_S = d_T$). We directly copy: $E_S = E_T$. 
\textit{Case 2: Dimension expansion} ($d_S > d_T$). When upgrading to higher capacity, we construct an orthogonal projection matrix through QR decomposition \cite{gander1980algorithms}:
\begin{equation}
\begin{aligned}
R &\sim \mathcal{N}(0, 1)^{d_S \times d_S}, \quad Q, \_ = \text{QR}(R) \text{ where } Q \in \mathbb{R}^{d_S \times d_S} \\
W &= Q[:, 1:d_T]^T \in \mathbb{R}^{d_T \times d_S}, \quad E_S = E_T \cdot W
\end{aligned}
\end{equation} 
\textit{Case 3: Dimension reduction} ($d_S < d_T$). When computational constraints require a smaller model, we apply Principal Component Analysis (PCA) to maximally preserve variance \cite{abdi2010principal}:
\begin{equation}
\begin{aligned}
\mu &= \frac{1}{V}\sum_{i=1}^V E_T[i,:]^\top, \quad \bar{E}_T = E_T - \mathbf{1}_V \mu^\top \\
C &= \frac{1}{V}\bar{E}_T^\top \bar{E}_T \in \mathbb{R}^{d_T \times d_T}, \quad U, \Lambda, U^\top = \text{EigenDecomp}(C) \\
W &= U[:, 1:d_S] \in \mathbb{R}^{d_T \times d_S}, \quad E_S = \bar{E}_T \cdot W + \mathbf{1}_V \mu^\top
\end{aligned}
\end{equation}

\textit{Intuition.} The PCA projection preserves the principal directions of variation in the feature space. For feature embeddings, these directions capture the main semantic relationships (e.g., user preferences, item categories), while discarding noise in minor components.

\textbf{Theoretical Guarantees.} The following theorem shows that our projection method optimally preserves feature relationships; the full proof is given in Appendix~\ref{app:proof}.

\begin{theorem}[Optimal Embedding Projections]
\label{thm:embedding_projection}
Let $\bar{E}_T = E_T - \mathbf{1}_V \mu^\top$ be the centered teacher embedding matrix and $C = \frac{1}{V}\bar{E}_T^\top \bar{E}_T$ be its covariance matrix with eigendecomposition $C = U\Lambda U^\top$ where $\Lambda = \mathrm{diag}(\lambda_1, \ldots, \lambda_{d_T})$ and $\lambda_1 \geq \cdots \geq \lambda_{d_T} \geq 0$. Then:
\textit{(i)} When $d_S = d_T$, direct copying preserves all pairwise inner products exactly.
\textit{(ii)} When $d_S > d_T$, any orthogonal expansion $W$ satisfying $WW^\top = I_{d_T}$ preserves all pairwise inner products exactly.
\textit{(iii)} When $d_S < d_T$ with $d_S \leq \mathrm{rank}(\bar{E}_T)$, the PCA projection $W = U[:, 1:d_S]$ achieves the minimum inner product distortion over all rank-$d_S$ projections, with Gram matrix error $\|G_T - G_S\|_F^2 = V^2 \sum_{k=d_S+1}^{d_T} \lambda_k^2$, where $G_T = \bar{E}_T \bar{E}_T^\top$ and $G_S = \bar{E}_T W W^\top \bar{E}_T^\top$.
\end{theorem}

This theoretical foundation guarantees optimal preservation of semantic relationships between categorical features. The projection requires only $O(d_T^3)$ for eigendecomposition and $O(V \cdot d_T \cdot d_S)$ for applying the transformation, making it significantly more efficient than iterative distillation methods requiring multiple forward-backward passes over the entire parameter space.

\subsubsection{Progressive Interaction Network Distillation}
After embedding transfer, directly training the entire student model end-to-end poses significant risk: randomly initialized interaction networks produce noisy gradients that can rapidly deteriorate the transferred embedding structure. To preserve embedding quality while effectively training interaction networks, we employ a progressive two-phase strategy decoupling optimization of different components.

We initialize student embeddings using dimension-adaptive projection, providing strong initialization capturing teacher feature relationships. In the first phase, we freeze embeddings to preserve transferred relationships and train only interaction networks to align with teacher predictions:
\begin{equation}
\min_{\theta_S^{\text{net}}} \mathcal{L}_{\text{net}} = \mathcal{L}_{\text{BCE}}(p_S, y) + \lambda \mathcal{L}_{\text{KD}}(p_S, p_T)
\end{equation}
where $\lambda$ balances task learning and knowledge transfer. The knowledge distillation loss for binary CTR/CVR prediction is formulated as cross-entropy between teacher and student predictions \cite{hinton2015distilling}:
$\mathcal{L}_{\text{KD}} = -\frac{1}{B}\sum_{i=1}^B \left[ p_T \log p_S + (1-p_T) \log (1-p_S) \right]$.
By freezing embeddings, interaction networks learn effective feature interactions based on stable representations, avoiding co-adaptation where both components deteriorate simultaneously. 

Once interaction networks reach competent performance, we unfreeze all parameters for joint optimization:
\begin{equation}
\min_{\theta_S} \mathcal{L}_{\text{joint}} = \mathcal{L}_{\text{BCE}}(p_S, y) + \lambda \mathcal{L}_{\text{KD}}(p_S, p_T)
\end{equation}
Both components are now sufficiently trained to co-evolve without mutual degradation. This progressive strategy proves more stable than end-to-end training in transfer learning \cite{howard2018universal}.

\begin{figure}[t]
\centering
\includegraphics[width=\linewidth]{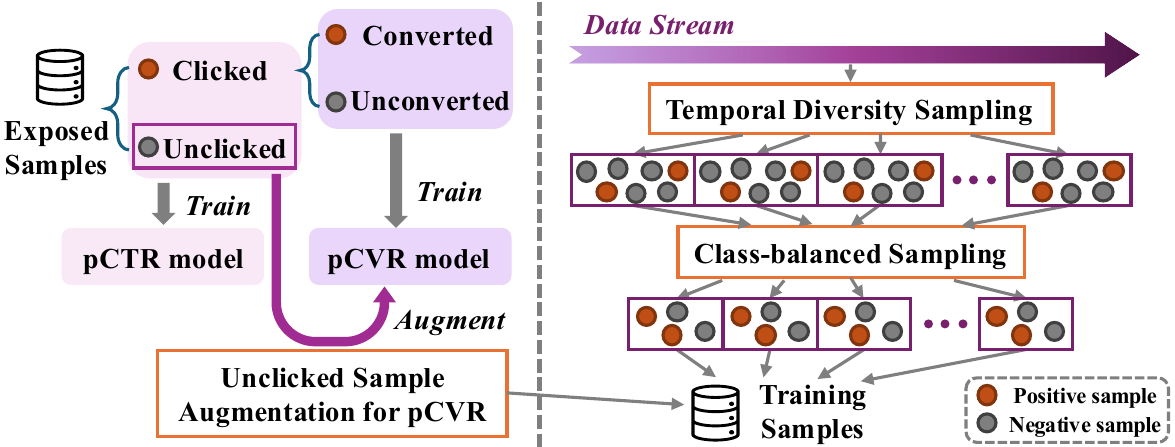}
\vspace{-1em}
\caption{Strategic distillation sample selection. Left: Sample space composition showing clicked samples for standard pCVR training and the augmentation with unclicked samples. Right: Sampling pipeline applying temporal diversity sampling followed by class-balanced sampling to construct the training set.}
\label{fig:sampling}
\vspace{-1em}
\end{figure}

\subsubsection{Strategic Distillation Sample Selection}
\label{sec:sampling}
While the teacher model is trained on the entire history dataset, using all samples of $\mathcal D_\text{train}$ for distillation is not necessary. Distillation on massive datasets incurs prohibitive computational costs and introduces redundancy, as many samples convey similar patterns. Instead, strategic sampling of informative and representative samples can achieve comparable knowledge transfer with reduced training time.
The key is to balance three critical factors: (1) \textit{information density}--samples that contain rich signals about user preferences and behaviors, (2) \textit{sample diversity}--coverage across different patterns to elicit comprehensive knowledge from the teacher, and (3) \textit{sample space completeness}--leveraging samples beyond observable labels through distillation. We employ a three-dimensional sampling strategy addressing these factors:

% \textbf{Class-balanced Sampling.} CTR/CVR data exhibits severe class imbalance with typical positive rates below 1\% \cite{he2014practical}. Positive samples (clicks) contain substantially richer information about user preferences than negative samples, as each click represents an active engagement revealing what users value, while non-clicks are often passive observations. Following insights from imbalanced learning \cite{chawla2002smote}, we oversample positive instances to achieve a target positive sampling ratio $r_{\text{pos}}$ and overall sampling ratio $r$:
% \begin{equation}\small
% \mathcal{D}_{\text{balanced}} = \Psi(\mathcal{D}_{\text{pos}}, r \cdot r_{\text{pos}}  |\mathcal{D}_{\text{train}}|) \cup \Psi(\mathcal{D}_{\text{neg}}, r  (1-r_{\text{pos}}) |\mathcal{D}_{\text{train}}|)
% \end{equation}
% where $\Psi(\mathcal{D}, n)$ denotes randomly sampling $n$ samples from dataset $\mathcal{D}$, and $\mathcal{D}_{\text{pos}}$ and $\mathcal{D}_{\text{neg}}$ represent the positive and negative samples in $\mathcal{D}_{\text{train}}$, respectively.
% This emphasizes learning from informative positive patterns while maintaining sufficient negative examples for discrimination.

\textbf{Class-balanced Sampling.} CTR/CVR data exhibits severe class imbalance with positive rates below 1\% \cite{he2014practical}. Positive samples contain richer information about user preferences, as clicks/conversions represent active engagement revealing user values, while non-clicks/non-conversions are passive observations. Following imbalanced learning insights \cite{chawla2002smote}, we oversample positive instances to achieve target positive ratio $r_{\text{pos}}$ and overall ratio $r$:
\begin{equation}\small
\mathcal{D}_{\text{balanced}} = \Psi(\mathcal{D}_{\text{pos}}, r \cdot r_{\text{pos}}  |\mathcal{D}_{\text{train}}|) \cup \Psi(\mathcal{D}_{\text{neg}}, r (1-r_{\text{pos}}) |\mathcal{D}_{\text{train}}|)
\end{equation}
where $\Psi(\mathcal{D}, n)$ denotes randomly sampling $n$ samples from $\mathcal{D}$, and $\mathcal{D}_{\text{pos}}$, $\mathcal{D}_{\text{neg}}$ represent positive and negative samples in $\mathcal{D}_{\text{train}}$.
This emphasizes learning from informative positive patterns while maintaining sufficient negatives for discrimination.

% \textbf{Temporal Diversity Sampling.} User behavior exhibits temporal patterns including daily/weekly periodicity, seasonal trends, and long-term evolution \cite{koren2009collaborative}. To capture these dynamics, we employ temporal diversity sampling with two motivations: (1) maintaining sample diversity to elicit more comprehensive knowledge from the teacher model across different temporal regimes, and (2) preserving temporal ordering to adapt to online learning scenarios where more recent samples naturally carry greater influence on model training. We partition data into $K$ temporal blocks and sample uniformly:
% \begin{equation}\small
% \mathcal{D}'_{\text{train}} = \bigcup_{k=1}^K \left[ \Psi(\mathcal{D}_{\text{pos}}^{(k)}, \frac{r  \cdot r_{\text{pos}} |\mathcal{D}_\text{train}|}{K}) \cup \Psi(\mathcal{D}_{\text{neg}}^{(k)}, \frac{r  (1-r_{\text{pos}})  |\mathcal{D}_\text{train}|}{K}) \right]
% \end{equation}
% where temporal blocks have equal time intervals. This ensures student is exposed to full spectrum of temporal dynamics captured by teacher, facilitating knowledge transfer across varing time periods.

\textbf{Temporal Diversity Sampling.} User behavior exhibits temporal patterns including periodicity, seasonal trends, and long-term evolution \cite{koren2009collaborative}. We employ temporal diversity sampling for two reasons: maintaining sample diversity to elicit comprehensive teacher knowledge across temporal regimes, and preserving temporal ordering where recent samples naturally carry greater influence. We partition data into $K$ equal temporal blocks and sample uniformly:
\begin{equation}\small
\mathcal{D}'_{\text{train}} = \bigcup_{k=1}^K \left[ \Psi(\mathcal{D}_{\text{pos}}^{(k)}, \frac{r \cdot r_{\text{pos}} |\mathcal{D}_\text{train}|}{K}) \cup \Psi(\mathcal{D}_{\text{neg}}^{(k)}, \frac{r (1-r_{\text{pos}}) |\mathcal{D}_\text{train}|}{K}) \right]
\end{equation}
This ensures the student is exposed to the full spectrum of temporal dynamics captured by the teacher, facilitating knowledge transfer across varying time periods.

% \textbf{Unclicked Sample Augmentation for pCVR.} 
% Traditional pCVR training uses only clicked samples (conversions as positives, non-conversions as negatives), ignoring unclicked exposures due to unobservable conversion labels. 
% This restriction severely limits the exploitable sample space, as clicked samples constitute only a small fraction of total impressions \cite{he2014practical}.
% CrossAdapt enables incorporating unclicked samples by using teacher predictions as pseudo-labels, thereby expanding the training distribution and enhancing knowledge transfer.
% We augment the pCVR training set by sampling unclicked exposures and assigning teacher-generated pseudo-CVR labels:
% \begin{equation}\small
% \mathcal{D}''_{\text{train}} = \mathcal{D}'_{\text{train}} \cup \{(\mathbf{x}_i, f_T^{\text{cvr}}(\mathbf{x}_i)) | \mathbf{x}_i \in \Psi(\mathcal{D}_{\text{unclick}}, r_{\text{unclick}} \cdot |\mathcal{D}'_{\text{train}}|)\}
% \end{equation}
% where $f_T^{\text{cvr}}$ denotes the teacher's pCVR output and $r_{\text{unclick}}$ controls the unclicked-to-clicked ratio.

\textbf{Unclicked Sample Augmentation for pCVR.} 
Traditional pCVR training uses only clicked samples (conversions as positives, non-conversions as negatives), ignoring unclicked exposures due to unobservable conversion labels. 
This severely limits the sample space, as clicks constitute only a small fraction of impressions \cite{he2014practical}.
We incorporates unclicked samples using teacher predictions as pseudo-labels, expanding the training distribution and enhancing knowledge transfer.
We augment the training set by sampling unclicked exposures with teacher-generated pseudo-CVR labels:
\begin{equation}\small
\mathcal{D}''_{\text{train}} = \mathcal{D}'_{\text{train}} \cup \{(\mathbf{x}_i, f_T^{\text{cvr}}(\mathbf{x}_i)) | \mathbf{x}_i \in \Psi(\mathcal{D}_{\text{unclick}}, r_{\text{unclick}} \cdot |\mathcal{D}'_{\text{train}}|)\}
\end{equation}
where $f_T^{\text{cvr}}$ denotes the teacher's pCVR output and $r_{\text{unclick}}$ controls the unclicked-to-clicked ratio.

\subsection{Online Adaptive Co-Distillation}
The online stage addresses \textbf{Challenge 2} through asymmetric teacher-student co-evolution and adaptive historical knowledge preservation for continuous adaptation to distribution shifts.

\subsubsection{Asymmetric Teacher-Student Co-evolution}

In online scenarios, data distributions and optimal parameters evolve continuously \cite{mcmahan2013ad}. Models must rapidly adapt to emerging patterns while preserving historical knowledge. However, knowledge distillation on streaming data often causes unstable supervision, as teacher models frequently update with new patterns \cite{kirkpatrick2017overcoming}.
To address this, we implement co-distillation where teacher and student train jointly but asymmetrically. The teacher serves as a stable knowledge anchor encoding robust patterns, while the student captures emerging trends. This asymmetry balances stability and adaptability: the teacher provides consistent supervision preventing the student from drifting, while gradually evolving to avoid obsolescence.

% \textbf{Student Model Update.} 
The student optimizes weighted task performance and knowledge distillation, updating at every step via gradient descent:
\begin{equation}
\theta_S^{t+1} = \theta_S^t - \eta_S \nabla_{\theta_S} [\mathcal{L}_\text{BCE}(p_S, y) + \lambda \cdot \mathcal{L}_\text{KD}(p_S, p_T)]
\end{equation}
Frequent updates enable quick response to distribution shifts, while distillation regularizes against overfitting to noisy patterns.

% \textbf{Teacher Model Stabilized Update.} 
The teacher accumulates gradients at every step but updates parameters less frequently for stable guidance. At step $t$, gradients accumulate recursively:
\begin{equation}
g_T^{t} = g_T^{t-1} + \nabla_{\theta_T} \mathcal{L}_\text{BCE}(p_T, y)
\end{equation}
Parameters update only at intervals of $\tau$, with accumulated gradients reset after each update:
\begin{equation}
\theta_T^{t+1}, g_T^{t+1} = \begin{cases}
\theta_T^t - \eta_T g^{t}_T, \; 0 & \text{if } t \bmod \tau = 0 \\
\theta_T^t, \; g^t_T & \text{otherwise}
\end{cases}
\end{equation}
where $\tau$ is the update interval (e.g., $\tau=10$). This accumulates gradients over $\tau$ steps, creating larger effective batch sizes for robust updates with reduced variance. The teacher updates only on task loss, maintaining its authoritative role. This asymmetric schedule ensures consistent targets across multiple student updates while allowing gradual assimilation of validated patterns, with $\tau$ controlling the stability-adaptability trade-off.

\subsubsection{Adaptive Historical Knowledge Preservation}
While dual model co-evolution provides stable online learning, knowledge distillation may suffer from \textit{catastrophic forgetting} in online scenarios. Inspired by experience replay \cite{rolnick2019experience}, we incorporate historical samples into current batches to prevent forgetting and increase diversity for robustness. However, this raises critical questions: \textit{when to incorporate historical samples?} and \textit{how to blend them with current data?} 

In online streaming user response prediction, distributions exhibit varying stability–from stationary periods where historical patterns remain valid, to rapid shifts where past knowledge becomes obsolete. Blindly mixing historical samples risks anchoring to outdated patterns, while ignoring history discards valuable knowledge that could accelerate distillation. 
This motivates an adaptive strategy that dynamically balances historical preservation and rapid adaptation based on distribution stability. 
The key insight is that \textit{batch diversity through historical samples} enhances distillation during stable periods by providing richer knowledge transfer signals, while \textit{focusing on current data consistency} accelerates adaptation during distribution shifts.

\textbf{Distribution Shift Detection.} We quantify distribution shift using distance metrics between feature distributions across consecutive time windows. Training data $\mathcal{D}^{\text{train}}$ is partitioned into $n$ consecutive windows $\{W_1, W_2, \ldots, W_n\}$ of equal size. For each adjacent pair $(W_i, W_{i+1})$, we measure distributional discrepancy using metrics such as Jensen-Shannon divergence, KL divergence, or Wasserstein distance, denoted $\mathbb{D}(\cdot, \cdot)$. For each feature dimension $j$, 
$\Delta_j^{(i)} = \mathbb{D}(P_j^{W_i}, P_j^{W_{i+1}})$, 
where $P_j^{W_i}$ and $P_j^{W_{i+1}}$ represent feature $j$'s distributions in windows $W_i$ and $W_{i+1}$. Distributions are estimated by type: numerical features use histogram binning with $b$ bins (typically $b=50$), while categorical features use empirical distributions from normalized counts. 
The overall shift metric aggregates across features and windows:
\begin{equation}
\Delta_\text{shift} = \frac{1}{n-1} \sum_{i=1}^{n-1} \left( \frac{1}{V} \sum_{j=1}^{V} \Delta_j^{(i)} \right)
\end{equation}
where $V$ is the total number of features. This averaging provides robust shift estimation by smoothing temporal fluctuations.

\textbf{Adaptive Enhancement Mechanism.} Based on $\Delta_\text{shift}$, the enhancement ratio $r_\text{enh}$ is computed as:
\begin{equation}\small
r_\text{enh} = \begin{cases}
0 ,\ \text{if } \Delta_\text{shift} > \theta_\text{high} \ \text{(major shift)} \\
k \cdot (1 - \Delta_\text{shift}/\theta_\text{high}),\  \text{if } \theta_\text{low} < \Delta_\text{shift} \leq \theta_\text{high} \ \text{(moderate shift)} \\
k ,\ \text{if } \Delta_\text{shift} \leq \theta_\text{low} \ \text{(stable)}
\end{cases}
\end{equation}
Under stable distributions ($\Delta_\text{shift} \leq \theta_\text{low}$), we enhance distillation by mixing historical samples to increase batch diversity; under shifting distributions ($\Delta_\text{shift} > \theta_\text{high}$), we focus on current data for rapid adaptation; under moderate shifts, we reduce enhancement proportionally. 
When $r_{\text{enh}} > 0$, we augment each streaming batch:
\begin{equation}
\mathcal{B}_\text{aug} = \mathcal{B}_t \cup \Psi(\mathcal{D}_\text{train}, r_\text{enh} \cdot |\mathcal{B}_t|)
\end{equation}
where $\mathcal{D}_\text{train}$ is the training pool, $|\mathcal{B}_t|$ is current batch size, and the augmented batch contains $r_\text{enh} \cdot |\mathcal{B}_t|$ total samples. This injects stable historical patterns without additional memory. The adaptive strategy balances enhanced distillation during stability with rapid adaptation during shifts.

\section{Offline Evaluation}
\subsection{Experimental Setup}
\label{sec:exp_setup}

Appendix~\ref{app:implementation_detail} provides detailed implementation information, including dataset descriptions, model configurations, hyperparameter settings, and experimental environment.

\subsubsection{Training Protocol} 
\label{sec:training_protocol}
The teacher model is trained on $\mathcal{D}^{\text{hist}} \cup \mathcal{D}^{\text{train}}$. The student conducts offline distillation on $\mathcal{D}^{\text{train}}$, followed by online co-distillation on $\mathcal{D}^{\text{online}}$, without access to $\mathcal{D}^{\text{hist}}$ (reflecting practical deployment constraints). During online learning, data is processed sequentially in temporal batches without shuffling.

\subsubsection{Baseline Methods}

We compare CrossAdapt with a range of baselines covering both non-distillation and conventional distillation methods.
\textbf{Non-Distillation Baselines.}  
\textit{(1) Scratch} trains the student from random initialization on $\mathcal{D}^{\text{train}} \cup \mathcal{D}^{\text{online}}$, serving as the standard no-transfer baseline.  
\textit{(2) Scratch-Online} trains only on $\mathcal{D}^{\text{online}}$ and was previously deployed in production.
%representing a data-limited case.  
\textit{(3) Full-Retrain} retrains the student from scratch on the entire dataset ($\mathcal{D}^{\text{hist}} \cup \mathcal{D}^{\text{train}} \cup \mathcal{D}^{\text{online}}$), providing an idealized upper bound that is infeasible in practice due to the inaccessibility of $\mathcal{D}^{\text{hist}}$.
\textbf{Traditional Knowledge Distillation.}  
\textit{(4) Vanilla-KD}~\cite{hinton2015distilling} applies standard distillation using KL divergence between teacher and student logits with temperature scaling.  
\textit{(5) Hint-KD}~\cite{romero2014fitnets} aligns intermediate representations via additional projection layers, focusing on embedding table alignment.  
\textit{(6) RKD}~\cite{park2019relational} preserves pairwise relational structures by matching distances and angles among samples.

\textbf{CrossAdapt Variants.}  
To analyze data sampling strategies in the offline stage, \textit{CrossAdapt-Full} trains on the entire $\mathcal{D}^{\text{train}}$, while \textit{CrossAdapt-Sample} uses a strategically selected 10\% subset (Sec.~\ref{sec:sampling}) to assess efficiency and performance under limited data.

\subsubsection{Evaluation Metrics}

We evaluate CTR prediction performance using two standard metrics.
\textbf{AUC} (Area Under the ROC Curve)~\cite{fawcett2006introduction,cheng2016wide} measures the probability that a model ranks a randomly chosen positive sample higher than a negative one. It is robust to class imbalance and serves as the primary indicator of ranking quality–higher AUC denotes stronger discriminative ability in capturing user preferences.
\textbf{LogLoss} (Logarithmic Loss)~\cite{wang2017deep,zhu2020ensembled} evaluates the deviation between predicted probabilities and ground-truth labels. As accurate probability calibration is essential for bidding and traffic allocation, LogLoss is widely used as a key business metric, where lower values indicate better calibration.
Following CTR research~\cite{wang2017deep,cheng2016wide}, changes at the 0.001 level in AUC or LogLoss are considered statistically and practically significant.

\begin{table*}[t]
\centering
\caption{Performance comparison on public datasets. Best student results (excluding Full-Retrain) in \textbf{bold}, second best \underline{underlined}. $\Delta$ denotes improvement over Scratch baseline (absolute difference for AUC/LogLoss, relative percentage for Time). The teacher model uses an MLP network with embedding dimension 8; students can be architecture-matched or architecture-mismatched.}
\vspace{-1em}
\label{tab:main_results_public}
\small
\resizebox{\linewidth}{!}{
\begin{tabular}{c|c|ccc|ccc|ccc}
\hline
\rowcolor{blue!15}
\textbf{Student} & & \multicolumn{3}{c|}{\textbf{Criteo}} & \multicolumn{3}{c|}{\textbf{Avazu}} & \multicolumn{3}{c}{\textbf{Criteo1T}} \\
\rowcolor{blue!15}
\textbf{Model} &\textbf{Method}  &AUC(\%)$\uparrow$ &Logloss$\downarrow$ &Time(s)$\downarrow$ &AUC(\%)$\uparrow$ &Logloss$\downarrow$ &Time(s)$\downarrow$&AUC(\%)$\uparrow$ &Logloss$\downarrow$ &Time(s)$\downarrow$ \\
\hline
\rowcolor{gray!10}
– &Teacher Model & 80.08±0.02 & 0.4526±0.0002 &– &74.58±0.07 &0.3961±0.0002 &– &77.32±0.02 &0.1331±0.0000 &– \\
\hline
\multirow{9}*{\makecell{MLP\\dim=8\\(matched)}} &Full-Retrain &80.19±0.01 &0.4523±0.0002 &159.7 &75.06±0.05 &0.3940±0.0005 &118.1 &77.59±0.03 &0.1328±0.0000 &300.7\\
\cline{2-11}
&Scratch  &79.90±0.03 &0.4556±0.0004 &70.0 &74.89±0.06 &0.3946±0.0007 
&66.0 &77.19±0.03 &0.1334±0.0000 &143.6\\
&Scratch-Online &78.79±0.10 &0.4667±0.0017 &\textbf{15.1} &73.29±0.52 &0.4032±0.0028 &\textbf{13.8}  &73.45±0.30 &0.1406±0.0008 &\textbf{19.2} \\
% \cline{2-11}
&Vanilla-KD  &80.15±0.01 &\underline{0.4525±0.0002} &74.0 &75.17±0.01 &0.3933±0.0002 &69.4 &77.44±0.02 & 0.1330±0.0000 &150.1\\
&Hint-KD &80.10±0.02 &0.4527±0.0001 &100.1  &74.79±0.08 &0.3951±0.0004 &86.9 &77.35±0.02 &0.1331±0.0000 &199.3\\
&RKD &80.10±0.01 &0.4527±0.0001 & 91.9 &75.05±0.04 &0.3937±0.0003 &81.4 &77.38±0.02 & 0.1330±0.0000 &186.2\\
\cline{2-11}
&CrossAdapt-Full &\textbf{80.29±0.01} & \textbf{0.4511±0.0002} & 133.6 & \textbf{75.38±0.08} & \textbf{0.3916±0.0004} & 118.5 & \textbf{77.66±0.02} & \textbf{0.1326±0.0000} & 269.0 \\
&CrossAdapt-Sample & \underline{80.17±0.02} & 0.4526±0.0005 & \underline{36.1} & \underline{75.32±0.12} & \underline{0.3920±0.0005} & \underline{25.0} & \underline{77.58±0.02} & \underline{0.1328±0.0000} & \underline{52.5} \\
\cline{2-11}
\rowcolor{orange!10}
&$\Delta$ (Sample) & \textcolor{blue}{+0.27} & \textcolor{blue}{-0.0030} & \textcolor{blue}{-48.4\%} & \textcolor{blue}{+0.43} & \textcolor{blue}{-0.0026} & \textcolor{blue}{-62.1\%} & \textcolor{blue}{+0.39} & \textcolor{blue}{-0.0006} & \textcolor{blue}{-63.4\%} \\
\hline
\multirow{9}*{\makecell{FiBiNET\\dim=16\\(mismatched)}} &Full-Retrain & 80.38±0.01 &0.4504±0.0002 &462.5 &75.42±0.06 &0.3917±0.0003 &327.6 &77.71±0.05 &0.1326±0.0001 &1101.8\\
\cline{2-11}
&Scratch & 80.08±0.01 &0.4544±0.0004 &381.1 &75.19±0.01 &0.3930±0.0004 &182.3 &77.59±0.03 &0.1328±0.0000 &300.7\\
&Scratch-Online &79.04±0.02 &0.4642±0.0003 &\textbf{52.9} &74.05±0.20 &0.3991±0.0012 &\textbf{37.4} &75.29±0.08 &0.1364±0.0002 &\textbf{61.4}\\
% \cline{2-11}
&Vanilla-KD &80.25±0.02 &0.4516±0.0002 & 260.2 &75.24±0.04 &0.3929±0.0003 &183.0 &77.47±0.02 &0.1329±0.0000 &524.4\\
&Hint-KD &80.17±0.01 &0.4522±0.0002 &292.7 &74.99±0.05 &0.3940±0.0002 &208.1 &77.38±0.03 &0.1330±0.0000 &586.7 \\
&RKD &80.25±0.02 &0.4516±0.0002 &266.0 &75.24±0.04 &0.3929±0.0003 &187.0 & 77.47±0.02 &0.1329±0.0000 &534.8 \\
\cline{2-11}
&CrossAdapt-Full &\textbf{80.33±0.01} & \textbf{0.4509±0.0001} & 739.7 & \textbf{75.47±0.07} & \textbf{0.3911±0.0004} & 322.5 & \textbf{77.67±0.02} & \textbf{0.1326±0.0000} & 975.1 \\
&CrossAdapt-Sample & \underline{80.27±0.01} & \underline{0.4515±0.0004} & \underline{109.6} & \textbf{75.47±0.10} & \underline{0.3912±0.0005} & \underline{69.9} & \underline{77.64±0.02} & \underline{0.1327±0.0000} & \underline{172.5} \\
\cline{2-11}
\rowcolor{orange!10}
&$\Delta$ (Sample) & \textcolor{blue}{+0.19} & \textcolor{blue}{-0.0029} & \textcolor{blue}{-71.2\%} & \textcolor{blue}{+0.28} & \textcolor{blue}{-0.0018} & \textcolor{blue}{-61.7\%} & \textcolor{blue}{+0.05} & \textcolor{blue}{-0.0001} & \textcolor{blue}{-42.6\%} \\
\hline
\end{tabular}}
\vspace{-1em}
\end{table*}

\subsection{Experimental Results}

\subsubsection{Overall Performance}
\label{sec:rq1}

Table~\ref{tab:main_results_public} shows CrossAdapt consistently achieves superior performance with significant efficiency gains over all baselines. CrossAdapt-Sample outperforms Scratch across datasets and architectures, achieving +0.27–0.43\% AUC improvements and 0.0006–0.0030 LogLoss reductions for MLP students, and +0.05–0.28\% AUC gains for FiBiNET students, while reducing training time by 43–71\%. These improvements are statistically significant despite using only 10\% of $\mathcal{D}^{\text{train}}$.
Compared with traditional distillation, CrossAdapt consistently performs better. Relative to Vanilla-KD, it attains comparable or higher AUC (+0.02–0.23\%) with over 50\% less training time; against Hint-KD and RKD, it improves AUC by +0.07–0.53\% while being 64–74\% faster. These results confirm CrossAdapt's two-stage design and strategic sampling effectively improve performance and efficiency, substantially reducing model switching cost. 

CrossAdapt-Full further approaches or surpasses the Full-Retrain upper bound, reaching 80.29\% AUC on Criteo (MLP) versus 80.19\% for the upper bound, showing its ability to recover historical knowledge through efficient transfer. In contrast, Scratch-Online suffers 1–4\% AUC drops, highlighting the importance of historical knowledge. CrossAdapt alleviates this by distilling past information while adapting to shifts, effectively mitigating switching issues.

Consistent gains on both MLP (matched) and FiBiNET (mismatched) students demonstrate strong cross-architecture generalization. Despite FiBiNET's larger dimensions and complex interactions, CrossAdapt consistently outperforms all baselines, validating that embedding projection effectively bridges heterogeneous spaces. More complex architectures yield superior performance, motivating upgrades. CrossAdapt naturally inherits these capacity-driven gains, demonstrating scalability with architectural sophistication.

\begin{figure}[t]
\centering
\includegraphics[width=0.9\linewidth]{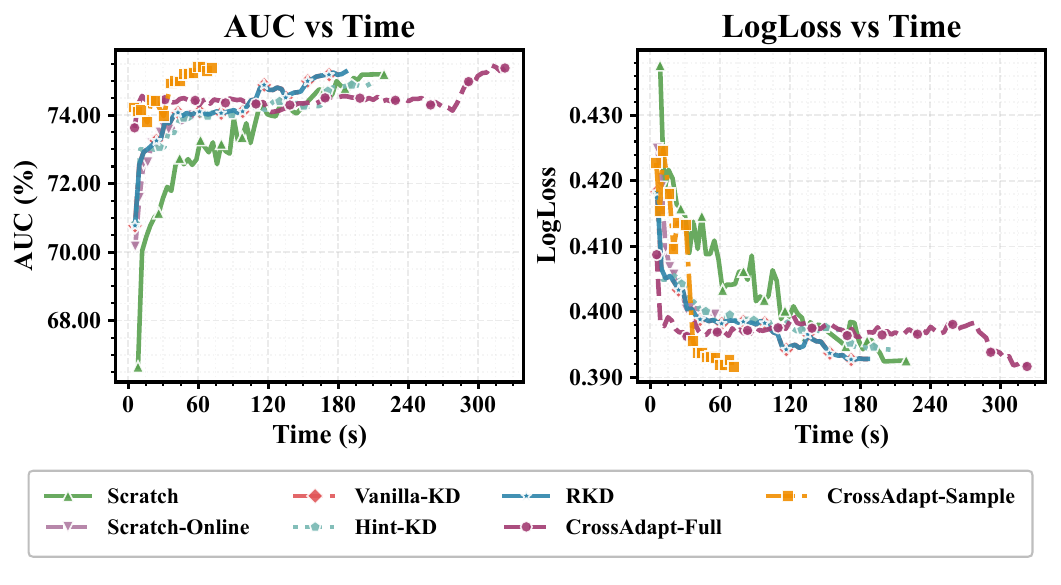}
\vspace{-1em}
\caption{Convergence curves on Avazu. CrossAdapt-Sample achieves faster convergence compared to baselines.}
\label{fig:time_avazu}
\vspace{-1em}
\end{figure}

\subsubsection{Convergence Analysis.}
Figure~\ref{fig:time_avazu} illustrates convergence on Avazu. CrossAdapt-Sample achieves the fastest convergence, reaching stable performance of $\geq$75.3\% AUC within 100 seconds, whereas Scratch requires over 180 seconds—a 44\% reduction in convergence time. Appendix~\ref{app:convergence_analysis} presents convergence on Criteo and Criteo1T with similar trends.

\begin{table*}[t]
\small
\centering
\caption{Ablation study on three datasets. We report AUC (\%), LogLoss, and absolute AUC changes (in percentage points) compared to the full CrossAdapt model. All configurations use FiBiNET as student architecture.}
\vspace{-1em}
\label{tab:ablation}
\small
\resizebox{\linewidth}{!}{
\begin{tabular}{l|ccc|ccc|ccc}
\hline
\rowcolor{blue!15}
& \multicolumn{3}{c|}{\textbf{Criteo}} & \multicolumn{3}{c|}{\textbf{Avazu}} & \multicolumn{3}{c}{\textbf{Criteo1T}} \\
\rowcolor{blue!15}
\textbf{Configuration} & AUC(\%)$\uparrow$ & $\Delta$ & LogLoss$\downarrow$ & AUC(\%)$\uparrow$ & $\Delta$ & LogLoss$\downarrow$ & AUC(\%)$\uparrow$ & $\Delta$ & LogLoss$\downarrow$ \\
\hline
\rowcolor{gray!10}
CrossAdapt-Sample & \textbf{80.27±0.01} & – & \textbf{0.4515±0.0004} & \textbf{75.47±0.10} & – & \textbf{0.3912±0.0005} & \textbf{77.64±0.02} & – & \textbf{0.1327±0.0000} \\
\hline
\rowcolor{orange!10}
\multicolumn{10}{c}{\textit{Offline Components}} \\
\hline
w/o Dimension-Adaptive Embedding Table Projection & 80.12±0.01 & \textcolor{blue}{-0.15} & 0.4529±0.0005 & 75.36±0.11 & \textcolor{blue}{-0.11} & 0.3917±0.0005 & 77.63±0.02 & \textcolor{blue}{-0.01} & \textbf{0.1327±0.0000} \\
w/o Progressive Interaction Network Distillation & 80.24±0.02 & \textcolor{blue}{-0.03} & 0.4517±0.0005 & 75.45±0.09 & \textcolor{blue}{-0.02} & 0.3913±0.0004 & 77.60±0.02 & \textcolor{blue}{-0.04} & \textbf{0.1327±0.0000} \\
w/o Strategic Distillation Sample Selection & 80.25±0.02 & \textcolor{blue}{-0.02} & 0.4517±0.0005 & 75.46±0.09 & \textcolor{blue}{-0.01} & \textbf{0.3912±0.0005} & 77.62±0.02 & \textcolor{blue}{-0.02} & \textbf{0.1327±0.0000} \\
\textbf{w/o Offline Stage} &80.04±0.01 & \textcolor{blue}{-0.23} &0.4535±0.0004 &75.24±0.07 & \textcolor{blue}{-0.23} &0.3923±0.0002 &77.45±0.02 & \textcolor{blue}{-0.19} &0.1331±0.0000 \\

\hline
\rowcolor{orange!10}
\multicolumn{10}{c}{\textit{Online Components}} \\
\hline
w/o Teacher-student Co-evolution & 80.20±0.01 & \textcolor{blue}{-0.07} & 0.4518±0.0003 & 75.20±0.04 & \textcolor{blue}{-0.27} & 0.3932±0.0002 & 77.44±0.02 & \textcolor{blue}{-0.20} & 0.1329±0.0000 \\
w/o Asymmetric Teacher-student Co-evolution & 80.20±0.01 & \textcolor{blue}{-0.07} & 0.4524±0.0005 & 75.46±0.02 & \textcolor{blue}{-0.01} & 0.3915±0.0002 & 77.59±0.02 & \textcolor{blue}{-0.05} & 0.1328±0.0000 \\
w/o Adaptive Historical Knowledge Preservation & 80.22±0.01 & \textcolor{blue}{-0.05} & 0.4520±0.0002 & 75.38±0.10  & \textcolor{blue}{-0.09} & 0.3916±0.0005  & 77.62±0.02 & \textcolor{blue}{-0.02} & \textbf{0.1327±0.0000} \\
\textbf{w/o Online Stage} &79.87±0.02 & \textcolor{blue}{-0.40} &0.4597±0.0012 &74.16±0.12 & \textcolor{blue}{-1.31} &0.4210±0.0034 &77.17±0.03 & \textcolor{blue}{-0.47} &0.1483±0.0013 \\
\hline
\end{tabular}
}
\vspace{-1em}
\end{table*}

\subsubsection{Ablation Study}
\label{sec:ablation}
Table~\ref{tab:ablation} reports ablation results on three datasets using FiBiNET as student, quantifying each component's contribution.
The embedding projection module shows the largest individual impact among offline components: removing it causes 0.01–0.15\% AUC drops, confirming its necessity for aligning heterogeneous embeddings across dimensions (8 vs. 16). Progressive training shows moderate importance with 0.02–0.04\% AUC reductions, indicating gradual unfreezing mitigates gradient interference. %While strategic sampling causes smaller AUC drops (0.01–0.02\%), it substantially improves efficiency. 
Online components demonstrate stronger effects under distribution shift. Removing teacher-student co-evolution (updating only the student) causes substantial degradation (up to 0.27\% AUC and 0.0020 LogLoss on Avazu), confirming outdated teacher guidance degrades performance. Eliminating asymmetric updates reduces AUC by up to 0.07\%, showing lower teacher update frequency stabilizes supervision. Removing adaptive enhancement results in 0.02–0.09\% AUC loss, demonstrating distribution-aware enhancement mitigates catastrophic forgetting and improves batch diversity. 
% Critically, removing the entire online stage causes the most severe degradation: 0.40–1.31\% AUC drops. 
% The large degradation on Avazu (-1.31\% AUC) suggests severe distribution shift, making continuous adaptation crucial.

Comparing entire stages reveals online adaptation (0.40–1.31\% AUC loss) is generally more critical than offline knowledge inheritance (0.19–0.23\% AUC loss), highlighting that continuous adaptation to evolving distributions is paramount. However, offline and online modules work synergistically: offline components establish efficient foundation, while online components balance adaptability and stability under shift. The combined framework achieves optimal performance by leveraging both effective inheritance and continuous adaptation.

\begin{table}[t]
\centering
\caption{Performance (AUC\%) of embedding dimension adaptation on Criteo. We compare CrossAdapt-Sample with and without Embedding Table Projection (ETP). $d_T$ and $d_S$ denote teacher and student embedding dimensions, respectively.}
\vspace{-1em}
\label{tab:dimension_adaptation}
\small
\resizebox{0.95\linewidth}{!}{
\begin{tabular}{lccccc}
\hline
\rowcolor{blue!15}
\textbf{$d_T \to d_S$} & \textbf{Teacher} & \textbf{Scratch} & \textbf{w/o ETP} & \textbf{w/ ETP}  &$\Delta$\\
\hline
\multicolumn{3}{l}{\textit{(Preservation)}} \\
8 $\to$ 8 & 80.08±0.02 & 80.09±0.00 & 80.12±0.01 & \textbf{80.27±0.01} & \textcolor{blue}{+0.18}\\
16 $\to$ 16 & 80.02±0.02 & 80.08±0.01 & 80.07±0.01 & \textbf{80.22±0.02} & \textcolor{blue}{+0.14}\\
\hline
\multicolumn{3}{l}{\textit{(Expansion)}} \\
8 $\to$ 16 & 80.08±0.02 & 80.08±0.01 & 80.12±0.01 & \textbf{80.27±0.01} & \textcolor{blue}{+0.19}\\
16 $\to$ 32 & 80.02±0.02 & 79.99±0.02 & 80.07±0.01 & \textbf{80.20±0.01} &  \textcolor{blue}{+0.21}\\
\hline
\multicolumn{3}{l}{\textit{(Reduction)}} \\
32 $\to$ 16 & 79.92±0.03 & 80.08±0.01 & 80.01±0.03 & \textbf{80.16±0.04} &  \textcolor{blue}{+0.08}\\
16 $\to$ 8 & 80.02±0.02 & 80.09±0.00 & 80.04±0.02 & \textbf{80.23±0.02} & \textcolor{blue}{+0.14}\\
\hline
\end{tabular}
}
\vspace{-1.5em}
\end{table}

\subsubsection{Robustness Analysis}
\label{sec:rq3}

Table~\ref{tab:dimension_adaptation} evaluates CrossAdapt's robustness across varying embedding dimensions. The Embedding Table Projection (ETP) module consistently improves over both Scratch and no-ETP baselines. In dimension preservation scenarios (8 $\to$ 8, 16 $\to$ 16), ETP yields +0.14–0.18\% AUC; expansion scenarios (8 $\to$ 16, 16 $\to$ 32) achieve +0.19–0.21\%, demonstrating ETP effectively leverages teacher knowledge when scaling dimensions; in reduction scenarios (32 $\to$ 16, 16 $\to$ 8), ETP provides +0.08–0.14\% improvements, indicating successful preservation of semantic relationships during compression. Overall, absolute AUC gains range from +0.08\% to +0.21\%, showing CrossAdapt's robustness across diverse dimension scenarios.
Additionally, Appendix~\ref{app:architecture_generalization} evaluates robustness across diverse architectures, and Appendix~\ref{app:additional_hyperparameter} analyzes hyperparameter effects, including sampling ratio, positive sampling ratio, distillation temperature, balance coefficient, historical augmentation ratio, and distribution shift detection metric.

\section{Online Deployment}
\label{sec:online_deployment}

To validate CrossAdapt's practical effectiveness, we deploy the framework on Tencent WeChat Channels advertising platform for CVR prediction in the re-ranking stage.

\subsection{Industrial Deployment Setup}
\label{sec:exp_online_setup}

% The deployment targets a production environment processing 25 million daily samples with 3,922 features. The model comprises 945 million embedding table keys on parameter servers and an interaction network on GPU servers. 
% The teacher model has been continuously optimized through online learning over several years.

The deployment targets a production environment processing $\sim$10M daily samples with $\sim$5K features. 
The model adopts a distributed architecture, where billion-level embedding tables are hosted on parameter servers, and the interaction network is deployed on GPU servers.
% The model comprises $\sim$1B embedding table keys stored on parameter servers and an interaction network deployed on GPU servers. 
The teacher model has been continuously optimized through long-term online learning.
Appendix~\ref{app:online_deploy_detail} provides implementation details.
We compare CrossAdapt against Vanilla-KD baseline:
\begin{itemize}[leftmargin=*,noitemsep]
    \item \textbf{CrossAdapt:} \textit{Stage 1}: One month of HDFS data (Dec 8, 2025 -- Jan 7, 2026) sampled at 20\% (6 days equivalent). Teacher frozen (cold backup from Jan 8, 2026 00:00), student updates. \textit{Stage 2}: 12 days of data (Jan 8--19, 2026) with teacher-student co-evolution.
    \item \textbf{Vanilla-KD:} 18 days of HDFS data (Jan 2--19, 2026), teacher from Jan 2, 2026 00:00 cold backup, both models update simultaneously.
\end{itemize}

\subsection{Online Experimental Results}

\begin{table}[t]\small
\centering
\caption{Online A/B test results comparing student models against production teacher on Tencent WeChat Channels (5-day test). $\Delta$ denotes relative change; negative AUC change and positive LogLoss/Bias indicate performance degradation.}
\vspace{-1em}
\label{tab:online_ab_test}
\begin{tabular}{lccc}
\hline
\rowcolor{blue!15}
\textbf{Method\quad} & \textbf{\quad$\Delta$AUC\quad} & \textbf{\quad$\Delta$LogLoss\quad} & \textbf{\quad$\Delta$pCVR Bias} \\
\hline
Vanilla-KD & -0.0024 & +0.2207 & +0.1026 \\
CrossAdapt & \textbf{-0.0010} & \textbf{+0.0844} & \textbf{+0.0440} \\
\hline
\rowcolor{orange!10}
\multicolumn{4}{l}{\textit{CrossAdapt reduces degradation by 58\%, 62\%, and 57\% respectively.}} \\
\hline
\end{tabular}
\vspace{-1em}
\end{table}

\begin{table}[t]\small
\centering
\caption{Impact of USAug on Vanilla-KD prediction quality. Models are evaluated on their prediction consistency with the production teacher model on exposed samples.}
\vspace{-1em}
\label{tab:usaug_results}
% \resizebox{\columnwidth}{!}{
\begin{tabular}{lcccc}
\hline
\rowcolor{blue!15}
\textbf{Model} & \textbf{Spearman} $\uparrow$ & \textbf{NDCG@5} $\uparrow$ & \textbf{NDCG@10} $\uparrow$ \\
\hline
w/o USAug & 0.5920 & 0.5379 & 0.5966  \\
w/ USAug & \textbf{0.9366} & \textbf{0.9887} & \textbf{0.9909}  \\
\hline
\rowcolor{orange!10}
Relative Improv. & +58.2\% & +83.8\% & +66.1\% \\
\hline
\end{tabular}
% }
\vspace{-1em}
\end{table}

We conducted a 5-day online A/B test on Tencent WeChat Channels advertising platform processing $\sim$10M daily samples. Table~\ref{tab:online_ab_test} compares student models against the production teacher model.
CrossAdapt substantially outperforms Vanilla-KD across all metrics, reducing AUC degradation by 58\% (from -0.0024 to -0.0010), LogLoss increase by 62\% (from +0.2207 to +0.0844), and pCVR bias by 57\% (from +0.1026 to +0.0440). The bias reduction is particularly valuable, as biased conversion predictions directly impact bidding strategies and budget allocation in real-time advertising.
CrossAdapt successfully reduces model switching costs in large-scale industrial systems, enabling faster deployment of improved model architectures while maintaining competitive performance.

For this pCVR scenario, we specifically evaluate the Unclicked Sample Augmentation (USAug) module (Sec.~\ref{sec:sampling}) on Vanilla-KD. Table~\ref{tab:usaug_results} demonstrates USAug's effectiveness in addressing sample selection bias.
Without USAug, the student model trained solely on clicked samples exhibits significant prediction bias on exposed items, achieving only 0.5920 Spearman correlation with the production teacher. Augmenting the training set with unclicked samples dramatically improves alignment: Spearman correlation increases by 58.2\% to 0.9366, while NDCG@5 and NDCG@10 improve by 83.8\% and 66.1\%. These results validate that expanding the training distribution to cover unclicked exposures mitigates estimation bias and enhances knowledge transfer fidelity across the full item space.

% \section{Online Deployment}

% \subsection{Industrial Deployment}

% We further evaluate CrossAdapt on a large-scale social media platform, which contains user interaction logs with billions of samples. The teacher model has been continuously optimized through online learning over several years. This dataset reflects real-world deployment challenges at industrial scale.

% \subsection{A/B Test Results}
% We evaluate CrossAdapt on an industrial dataset, serving billions of users. Offline evaluation after 5 days of online training shows that CrossAdapt-Sample achieves XX.XX\% AUC, outperforming Scratch by +XX\% and traditional KD methods by +XX–XX\%, with a XX\% LogLoss reduction, indicating improved probability calibration for bidding and traffic allocation.
% In 14-day online A/B testing with 5\% traffic per bucket, CrossAdapt yields a +XX\% CTR lift, substantially exceeding Vanilla-KD's +XX\%. User engagement also improves: dwell time rises by XXs (XXs $\to$ XXs), completion rate by XX percentage points (XX\% $\to$ XX\%), and share rate by XX points (XX\% $\to$ XX\%).

% \finding{\textbf{RQ5:} \textbf{CrossAdapt demonstrates strong performance and cost-effectiveness in industrial deployment.} On WeChat's billion-scale dataset, it achieves +1.72\% AUC and bridges 97.3\% of the teacher-student gap. Online A/B testing confirms +1.47\% CTR lift with statistical significance ($p<0.001$). The 49.1\% training time reduction translates to \$10,000-15,000 annual savings per model, validating production readiness.}

\section{Related Work} 

Knowledge distillation (KD) focuses on transferring learned representations from a large teacher model to a smaller student model by aligning output distributions or intermediate feature spaces \cite{hinton2015distilling, romero2014fitnets, zagoruyko2016paying}.  
Knowledge transfer, in contrast, more broadly refers to reusing information (e.g., parameters, embeddings, or relational structures) across different models or tasks, without requiring strict teacher-student supervision \cite{pan2009survey}.
In recommendation and user response prediction, KD has been widely used to accelerate inference by compressing deep interaction networks \cite{tang2018ranking, zhu2020ensemble}, to improve model generalization, robustness, and efficiency \cite{liu2020general,kang2021topology}.  
However, embedding tables, which dominate pCTR/pCVR model parameters, are hard to distill due to their massive scale.  
Existing KD methods typically focus on interaction networks, neglecting the embedding component that stores most of the model's knowledge \cite{zhu2020ensemble}.   
% However, the heterogeneity of architectures between teacher and student–different embedding dimensions or network topologies–makes layer-wise or feature-level distillation infeasible \cite{romero2014fitnets}.   
% Third, iterative distillation across billions of samples is computationally prohibitive in industrial systems \cite{ranganathan2025zero}.  
% To overcome these limitations, our work employs a dimension-adaptive, parameter-free embedding projection combined with progressive interaction network distillation, enabling efficient and architecture-agnostic knowledge transfer.

Online learning is widely used in large-scale user response prediction systems to handle continuously changing user behaviors, item pools, and contextual features \cite{mcmahan2013ad, he2014practical}.  
Unlike offline retraining, it updates models incrementally with streaming data for timely adaptation.  
Common approaches include incremental SGD or adaptive optimizers such as Adagrad and FTRL \cite{duchi2011adaptive, mcmahan2013ad}, as well as hybrid architectures that combine short-term and long-term user behavior modeling to balance stability and reactivity \cite{zhou2019deep, pi2019practice}.  
Experience replay mechanisms are employed to mitigate catastrophic forgetting in continual learning settings \cite{rolnick2019experience, chaudhry2019continual}.
A central challenge lies in detecting and adapting to distribution shifts.  
Statistical metrics such as KL divergence are often used to monitor data drift \cite{lipton2018detecting, rabanser2019failing}, while adaptive learning rates, continual learning, or regularization-based methods adjust model parameters accordingly \cite{kirkpatrick2017overcoming, zenke2017continual}.  
% Building on these ideas, our method introduces a divergence-aware adaptation mechanism that explicitly detects distributional shifts and performs adaptive knowledge transfer, achieving better stability–adaptability balance in non-stationary user response prediction environments.

\section{Key Insights and Conclusions}

We summarize the key design principles of this work.
\textbf{(1) Two-Stage Training Paradigm.}
The separation of offline and online phases is essential. Technically, freezing the teacher during offline distillation prevents it from overfitting sampled data while underfitting unsampled regions, which would create imbalanced knowledge undermining transfer quality. From a deployment perspective, offline training inevitably incurs delays, rendering the teacher outdated by completion. The online co-evolution phase enables both models to catch up with current data distributions for safe student deployment. 
\textbf{(2) Sample Organization and Curation.}
Sample diversity critically impacts distillation effectiveness by eliciting teacher knowledge more comprehensively, accelerating alignment, and improving robustness. Our framework adopts a unified design philosophy across both stages: offline class-balanced sampling increases information density, temporal diversity captures behavioral dynamics, and unclicked augmentation expands training space; online adaptive enhancement adjusts composition based on distribution shifts. These strategies optimize transfer efficiency through enriched sample informativeness and coverage.
\textbf{(3) Fast Embedding Inheritance.}
Dimension-adaptive projection addresses a fundamental bottleneck in embedding table training: since each batch activates only observed feature embeddings, random initialization requires prohibitive sample volumes to sufficiently train all embeddings. Our mathematical mapping directly inherits teacher embeddings, bypassing lengthy warm-up periods.

This paper addresses model switching costs in large-scale user response prediction through CrossAdapt, a two-stage knowledge transfer framework that decouples offline inheritance from online adaptation. The dimension-adaptive projection enables rapid, near-lossless embedding table transfer, bypassing iterative training on massive parameters. The asymmetric co-distillation mechanism dynamically balances stability and adaptability by monitoring distribution shifts. Extensive experiments demonstrate 43-71\% training time reduction with consistent performance gains, while industrial deployment on Tencent WeChat Channels validates practical effectiveness at scale. CrossAdapt provides a principled solution for deploying advanced architectures, enabling practitioners to leverage architectural innovations without prohibitive switching costs.

% While our unclicked sample augmentation demonstrates clear benefits for pCVR training, the current implementation uses a fixed sampling ratio and simple random selection. Future work could explore more sophisticated strategies, including dynamically adjusting the clicked-to-unclicked ratio based on teacher confidence, designing uncertainty-based or diversity-driven sampling policies for unclicked samples, and incorporating pseudo-label quality assessment to weight samples differentially during training.

\bibliographystyle{ACM-Reference-Format}
\bibliography{ref}

@article{abdi2010principal,
  title={Principal component analysis},
  author={Abdi, Herv{\'e} and Williams, Lynne J},
  journal={Wiley interdisciplinary reviews: computational statistics},
  volume={2},
  number={4},
  pages={433--459},
  year={2010},
  publisher={Wiley Online Library}
}

@inproceedings{cheng2016wide,
  title={Wide \& deep learning for recommender systems},
  author={Cheng, Heng-Tze and Koc, Levent and Harmsen, Jeremiah and Shaked, Tal and Chandra, Tushar and Aradhye, Hrishi and Anderson, Glen and Corrado, Greg and Chai, Wei and Ispir, Mustafa and others},
  booktitle={Proceedings of the 1st workshop on deep learning for recommender systems},
  pages={7--10},
  year={2016}
}

@article{fawcett2006introduction,
  title={An introduction to ROC analysis},
  author={Fawcett, Tom},
  journal={Pattern recognition letters},
  volume={27},
  number={8},
  pages={861--874},
  year={2006},
  publisher={Elsevier}
}

@inproceedings{he2014practical,
  title={Practical lessons from predicting clicks on ads at facebook},
  author={He, Xinran and Pan, Junfeng and Jin, Ou and Xu, Tianbing and Liu, Bo and Xu, Tao and Shi, Yanxin and Atallah, Antoine and Herbrich, Ralf and Bowers, Stuart and others},
  booktitle={Proceedings of the eighth international workshop on data mining for online advertising},
  pages={1--9},
  year={2014}
}

@article{hinton2015distilling,
  title={Distilling the knowledge in a neural network},
  author={Hinton, Geoffrey and Vinyals, Oriol and Dean, Jeff},
  journal={arXiv preprint arXiv:1503.02531},
  year={2015}
}

@article{kirkpatrick2017overcoming,
  title={Overcoming catastrophic forgetting in neural networks},
  author={Kirkpatrick, James and Pascanu, Razvan and Rabinowitz, Neil and Veness, Joel and Desjardins, Guillaume and Rusu, Andrei A and Milan, Kieran and Quan, John and Ramalho, Tiago and Grabska-Barwinska, Agnieszka and others},
  journal={Proceedings of the national academy of sciences},
  volume={114},
  number={13},
  pages={3521--3526},
  year={2017},
  publisher={National Academy of Sciences}
}

@inproceedings{mcmahan2013ad,
  title={Ad click prediction: a view from the trenches},
  author={McMahan, H Brendan and Holt, Gary and Sculley, David and Young, Michael and Ebner, Dietmar and Grady, Julian and Nie, Lan and Phillips, Todd and Davydov, Eugene and Golovin, Daniel and others},
  booktitle={Proceedings of the 19th ACM SIGKDD international conference on Knowledge discovery and data mining},
  pages={1222--1230},
  year={2013}
}

@inproceedings{park2019relational,
  title={Relational knowledge distillation},
  author={Park, Wonpyo and Kim, Dongju and Lu, Yan and Cho, Minsu},
  booktitle={Proceedings of the IEEE/CVF Conference on Computer Vision and Pattern Recognition},
  pages={3967--3976},
  year={2019}
}

@inproceedings{romero2014fitnets,
  title={Fitnets: Hints for thin deep nets},
  author={Romero, Adriana and Ballas, Nicolas and Kahou, Samira Ebrahimi and Chassang, Antoine and Gatta, Carlo and Bengio, Yoshua},
  booktitle={International Conference on Learning Representations},
  year={2015}
}

@inproceedings{song2019autoint,
  title={Autoint: Automatic feature interaction learning via self-attentive neural networks},
  author={Song, Weiping and Shi, Chence and Xiao, Zhiping and Duan, Zhijian and Xu, Yewen and Zhang, Ming and Tang, Jian},
  booktitle={Proceedings of the 28th ACM international conference on information and knowledge management},
  pages={1161--1170},
  year={2019}
}

@incollection{wang2017deep,
  title={Deep \& cross network for ad click predictions},
  author={Wang, Ruoxi and Fu, Bin and Fu, Gang and Wang, Mingliang},
  booktitle={Proceedings of the ADKDD'17},
  pages={1--7},
  year={2017}
}

@article{wang2020practical,
  title={A practical incremental method to train deep ctr models},
  author={Wang, Yichao and Guo, Huifeng and Tang, Ruiming and Liu, Zhirong and He, Xiuqiang},
  journal={arXiv preprint arXiv:2009.02147},
  year={2020}
}

@inproceedings{zhu2020ensembled,
  title={Ensembled CTR prediction via knowledge distillation},
  author={Zhu, Jieming and Liu, Jinyang and Li, Weiqi and Lai, Jincai and He, Xiuqiang and Chen, Liang and Zheng, Zibin},
  booktitle={Proceedings of the 29th ACM international conference on information \& knowledge management},
  pages={2941--2958},
  year={2020}
}

@inproceedings{lai2023adaembed,
  title={$\{$AdaEmbed$\}$: Adaptive embedding for $\{$Large-Scale$\}$ recommendation models},
  author={Lai, Fan and Zhang, Wei and Liu, Rui and Tsai, William and Wei, Xiaohan and Hu, Yuxi and Devkota, Sabin and Huang, Jianyu and Park, Jongsoo and Liu, Xing and others},
  booktitle={17th USENIX Symposium on Operating Systems Design and Implementation (OSDI 23)},
  pages={817--831},
  year={2023}
}

@article{liu2021learnable,
  title={Learnable embedding sizes for recommender systems},
  author={Liu, Siyi and Gao, Chen and Chen, Yihong and Jin, Depeng and Li, Yong},
  booktitle={International Conference on Learning Representations},
  year={2021}
}

@inproceedings{zhou2018deep,
author = {Zhou, Guorui and Zhu, Xiaoqiang and Song, Chenru and Fan, Ying and Zhu, Han and Ma, Xiao and Yan, Yanghui and Jin, Junqi and Li, Han and Gai, Kun},
title = {Deep Interest Network for Click-Through Rate Prediction},
year = {2018},
booktitle = {Proceedings of the 24th ACM SIGKDD International Conference on Knowledge Discovery \& Data Mining},
pages = {1059--1068},
numpages = {10},
}

@article{zhang2021deep,
author = {Zhang, Shuai and Yao, Lina and Sun, Aixin and Tay, Yi},
title = {Deep Learning Based Recommender System: A Survey and New Perspectives},
year = {2019},
volume = {52},
number = {1},
journal = {ACM Comput. Surv.},
month = feb,
articleno = {5},
numpages = {38}
}

@article{guo2017deepfm,
  title={DeepFM: a factorization-machine based neural network for CTR prediction},
  author={Guo, Huifeng and Tang, Ruiming and Ye, Yunming and Li, Zhenguo and He, Xiuqiang},
  year = {2017},
  booktitle = {Proceedings of the 26th International Joint Conference on Artificial Intelligence},
  pages = {1725--1731},
  numpages = {7},
  series = {IJCAI'17}
}

@inproceedings{zhang2020retrain,
  title={How to retrain recommender system? A sequential meta-learning method},
  author={Zhang, Yang and Feng, Fuli and Wang, Chenxu and He, Xiangnan and Wang, Meng and Li, Yan and Zhang, Yongdong},
  booktitle={Proceedings of the 43rd International ACM SIGIR Conference on Research and Development in Information Retrieval},
  pages={1479--1488},
  year={2020}
}

@article{ba2014deep,
  title={Do deep nets really need to be deep?},
  author={Ba, Jimmy and Caruana, Rich},
  journal={Advances in neural information processing systems},
  volume={27},
  year={2014}
}

@article{zagoruyko2016paying,
  title={Paying more attention to attention: Improving the performance of convolutional neural networks via attention transfer},
  author={Zagoruyko, Sergey and Komodakis, Nikos},
  booktitle={International Conference on Learning Representations},
  year={2017}
}

@article{pan2009survey,
  title={A survey on transfer learning},
  author={Pan, Sinno Jialin and Yang, Qiang},
  journal={IEEE Transactions on knowledge and data engineering},
  volume={22},
  number={10},
  pages={1345--1359},
  year={2009},
  publisher={IEEE}
}

@inproceedings{tang2018ranking,
  title={Ranking distillation: Learning compact ranking models with high performance for recommender system},
  author={Tang, Jiaxi and Wang, Ke},
  booktitle={Proceedings of the 24th ACM SIGKDD international conference on knowledge discovery \& data mining},
  pages={2289--2298},
  year={2018}
}

@inproceedings{zhu2020ensemble,
author = {Zhu, Jieming and Liu, Jinyang and Li, Weiqi and Lai, Jincai and He, Xiuqiang and Chen, Liang and Zheng, Zibin},
title = {Ensembled CTR Prediction via Knowledge Distillation},
year = {2020},
booktitle = {Proceedings of the 29th ACM International Conference on Information \& Knowledge Management},
pages = {2941--2958},
numpages = {18},
series = {CIKM '20}
}

@article{duchi2011adaptive,
  title={Adaptive subgradient methods for online learning and stochastic optimization},
  author={Duchi, John and Hazan, Elad and Singer, Yoram},
  journal={Journal of machine learning research},
  volume={12},
  number={7},
  year={2011}
}

@inproceedings{zhou2019deep,
  title={Deep interest evolution network for click-through rate prediction},
  author={Zhou, Guorui and Mou, Na and Fan, Ying and Pi, Qi and Bian, Weijie and Zhou, Chang and Zhu, Xiaoqiang and Gai, Kun},
  booktitle={Proceedings of the AAAI conference on artificial intelligence},
  volume={33},
  number={01},
  pages={5941--5948},
  year={2019}
}

@inproceedings{pi2019practice,
  title={Practice on long sequential user behavior modeling for click-through rate prediction},
  author={Pi, Qi and Bian, Weijie and Zhou, Guorui and Zhu, Xiaoqiang and Gai, Kun},
  booktitle={Proceedings of the 25th ACM SIGKDD international conference on knowledge discovery \& data mining},
  pages={2671--2679},
  year={2019}
}

@article{rolnick2019experience,
  title={Experience replay for continual learning},
  author={Rolnick, David and Ahuja, Arun and Schwarz, Jonathan and Lillicrap, Timothy and Wayne, Gregory},
  journal={Advances in neural information processing systems},
  volume={32},
  year={2019}
}

@inproceedings{chaudhry2019continual,
  title={Continual learning with tiny episodic memories},
  author={Chaudhry, Arslan and Rohrbach, Marcus and Elhoseiny, Mohamed and Ajanthan, Thalaiyasingam and Dokania, P and Torr, P and Ranzato, M},
  booktitle={Workshop on Multi-Task and Lifelong Reinforcement Learning},
  year={2019}
}

@inproceedings{lipton2018detecting,
  title={Detecting and correcting for label shift with black box predictors},
  author={Lipton, Zachary and Wang, Yu-Xiang and Smola, Alexander},
  booktitle={International conference on machine learning},
  pages={3122--3130},
  year={2018},
  organization={PMLR}
}

@article{rabanser2019failing,
  title={Failing loudly: An empirical study of methods for detecting dataset shift},
  author={Rabanser, Stephan and G{\"u}nnemann, Stephan and Lipton, Zachary},
  journal={Advances in Neural Information Processing Systems},
  volume={32},
  year={2019}
}

@inproceedings{zenke2017continual,
  title={Continual learning through synaptic intelligence},
  author={Zenke, Friedemann and Poole, Ben and Ganguli, Surya},
  booktitle={International conference on machine learning},
  pages={3987--3995},
  year={2017},
  organization={PMLR}
}

@inproceedings{liu2020general,
  title={A general knowledge distillation framework for counterfactual recommendation via uniform data},
  author={Liu, Dugang and Cheng, Pengxiang and Dong, Zhenhua and He, Xiuqiang and Pan, Weike and Ming, Zhong},
  booktitle={Proceedings of the 43rd international ACM SIGIR conference on research and development in information retrieval},
  pages={831--840},
  year={2020}
}

@inproceedings{kang2021topology,
  title={Topology distillation for recommender system},
  author={Kang, SeongKu and Hwang, Junyoung and Kweon, Wonbin and Yu, Hwanjo},
  booktitle={Proceedings of the 27th ACM SIGKDD Conference on Knowledge Discovery \& Data Mining},
  pages={829--839},
  year={2021}
}

@article{naumov2019deep,
  title={Deep learning recommendation model for personalization and recommendation systems},
  author={Naumov, Maxim and Mudigere, Dheevatsa and Shi, Hao-Jun Michael and Huang, Jianyu and Sundaraman, Narayanan and Park, Jongsoo and Wang, Xiaodong and Gupta, Udit and Wu, Carole-Jean and Azzolini, Alisson G and others},
  journal={arXiv preprint arXiv:1906.00091},
  year={2019}
}

@article{mikolov2013distributed,
  title={Distributed representations of words and phrases and their compositionality},
  author={Mikolov, Tomas and Sutskever, Ilya and Chen, Kai and Corrado, Greg S and Dean, Jeff},
  journal={Advances in neural information processing systems},
  volume={26},
  year={2013}
}

@article{gander1980algorithms,
  title={Algorithms for the QR decomposition},
  author={Gander, Walter},
  journal={Res. Rep},
  volume={80},
  number={02},
  pages={1251--1268},
  year={1980}
}

@article{howard2018universal,
  title={Universal language model fine-tuning for text classification},
  author={Howard, Jeremy and Ruder, Sebastian},
  journal={arXiv preprint arXiv:1801.06146},
  year={2018}
}

@article{chawla2002smote,
  title={SMOTE: synthetic minority over-sampling technique},
  author={Chawla, Nitesh V and Bowyer, Kevin W and Hall, Lawrence O and Kegelmeyer, W Philip},
  journal={Journal of artificial intelligence research},
  volume={16},
  pages={321--357},
  year={2002}
}

@inproceedings{koren2009collaborative,
  title={Collaborative filtering with temporal dynamics},
  author={Koren, Yehuda},
  booktitle={Proceedings of the 15th ACM SIGKDD international conference on Knowledge discovery and data mining},
  pages={447--456},
  year={2009}
}

@article{xiao2017attentional,
  title={Attentional factorization machines: Learning the weight of feature interactions via attention networks},
  author={Xiao, Jun and Ye, Hao and He, Xiangnan and Zhang, Hanwang and Wu, Fei and Chua, Tat-Seng},
  journal={arXiv preprint arXiv:1708.04617},
  year={2017}
}

@inproceedings{huang2019fibinet,
  title={FiBiNET: combining feature importance and bilinear feature interaction for click-through rate prediction},
  author={Huang, Tongwen and Zhang, Zhiqi and Zhang, Junlin},
  booktitle={Proceedings of the 13th ACM conference on recommender systems},
  pages={169--177},
  year={2019}
}

@inproceedings{covington2016deep,
  title={Deep neural networks for youtube recommendations},
  author={Covington, Paul and Adams, Jay and Sargin, Emre},
  booktitle={Proceedings of the 10th ACM conference on recommender systems},
  pages={191--198},
  year={2016}
}

\clearpage
\appendix

\section{Theoretical Proof (Sec.~\ref{sec:embedding_projection})}
\label{app:proof}

\textbf{Theorem~\ref{thm:embedding_projection}} (Optimal Embedding Projections).
Let $\bar{E}_T = E_T - \mathbf{1}_V \mu^\top$ be the centered teacher embedding matrix and $C = \frac{1}{V}\bar{E}_T^\top \bar{E}_T$ be its covariance matrix with eigendecomposition $C = U\Lambda U^\top$ where $\Lambda = \mathrm{diag}(\lambda_1, \ldots, \lambda_{d_T})$ and $\lambda_1 \geq \cdots \geq \lambda_{d_T} \geq 0$. Then:
\textit{(i)} When $d_S = d_T$, direct copying preserves all pairwise inner products exactly.
\textit{(ii)} When $d_S > d_T$, any orthogonal expansion $W$ satisfying $WW^\top = I_{d_T}$ preserves all pairwise inner products exactly.
\textit{(iii)} When $d_S < d_T$ with $d_S \leq \mathrm{rank}(\bar{E}_T)$, the PCA projection $W = U[:, 1:d_S]$ achieves the minimum inner product distortion over all rank-$d_S$ projections, with Gram matrix error $\|G_T - G_S\|_F^2 = V^2 \sum_{k=d_S+1}^{d_T} \lambda_k^2$, where $G_T = \bar{E}_T \bar{E}_T^\top$ and $G_S = \bar{E}_T W W^\top \bar{E}_T^\top$.

\begin{proof}
We analyze each case with detailed derivations, focusing on the preservation of inner product relationships.

\paragraph{Case 1: Dimension preservation} 
When $d_S = d_T$ and $E_S = E_T$, for any categorical feature items $i, j \in [V]$, we have $\langle E_S[i,:], E_S[j,:] \rangle = \langle E_T[i,:], E_T[j,:] \rangle$, immediately preserving all pairwise inner products.

\paragraph{Case 2: Dimension expansion}
Consider an orthogonal matrix $Q \in \mathbb{R}^{d_S \times d_S}$ satisfying $Q^\top Q = I_{d_S}$. Extracting the first $d_T$ columns forms $A = Q[:, 1:d_T] \in \mathbb{R}^{d_S \times d_T}$ with orthonormal columns, yielding $A^\top A = I_{d_T}$. Defining $W = A^\top \in \mathbb{R}^{d_T \times d_S}$, we obtain $WW^\top = A^\top A = I_{d_T}$.

For the expanded embeddings $E_S = E_T W \in \mathbb{R}^{V \times d_S}$, the inner product between any two categorical feature items is:
\begin{equation}
\begin{aligned}
\langle E_S[i,:], E_S[j,:] \rangle &= (E_T[i,:] W)(E_T[j,:] W)^\top = E_T[i,:] W W^\top E_T[j,:]^\top \\
&= E_T[i,:] I_{d_T} E_T[j,:]^\top = \langle E_T[i,:], E_T[j,:] \rangle
\end{aligned}
\end{equation}
Thus, embedding expansion through orthogonal projection preserves all pairwise inner products exactly.

\paragraph{Case 3: Dimension reduction}
Let $r = \mathrm{rank}(\bar{E}_T) \leq \min\{V, d_T\}$ be the rank of the centered embedding matrix, and assume $d_S \leq r$ to ensure meaningful dimensionality reduction.

\textbf{SVD decomposition.} The thin SVD of $\bar{E}_T \in \mathbb{R}^{V \times d_T}$ is $\bar{E}_T = \widetilde{U} S R^\top$, where $\widetilde{U} \in \mathbb{R}^{V \times r}$ has orthonormal columns (left singular vectors), $S = \mathrm{diag}(\sigma_1, \ldots, \sigma_r) \in \mathbb{R}^{r \times r}$ contains singular values $\sigma_1 \geq \cdots \geq \sigma_r > 0$, and $R \in \mathbb{R}^{d_T \times r}$ has orthonormal columns (right singular vectors).

\textbf{Eigenvalue-singular value connection.} The empirical covariance matrix is
\begin{equation}
C = \frac{1}{V}\bar{E}_T^\top \bar{E}_T = \frac{1}{V} R S^\top \widetilde{U}^\top \widetilde{U} S R^\top = \frac{1}{V} R S^2 R^\top
\end{equation}
revealing that $\lambda_k = \sigma_k^2/V$ for $k = 1, \ldots, r$ and $\lambda_k = 0$ for $k = r+1, \ldots, d_T$, with eigenvectors $U = [R, U_{\perp}]$ where $U_{\perp}$ spans the null space of $C$.

\textbf{Optimality of PCA projection.} 
By the Eckart-Young-Mirsky theorem, the best rank-$d_S$ approximation to $\bar{E}_T$ in Frobenius norm keeps the top $d_S$ singular values, yielding reconstruction error
\begin{equation}
\|\bar{E}_T - \bar{E}_T WW^\top\|_F^2 = \sum_{k=d_S+1}^r \sigma_k^2 = V \sum_{k=d_S+1}^{d_T} \lambda_k
\end{equation}
minimized when $W = U[:, 1:d_S]$, the matrix of top $d_S$ eigenvectors of $C$. This projection also maximizes the retained variance $\mathrm{tr}(W^\top C W) = \sum_{k=1}^{d_S} \lambda_k$, with the fraction of variance retained being $\sum_{k=1}^{d_S} \lambda_k / \sum_{k=1}^{d_T} \lambda_k$, providing a practical metric for choosing $d_S$.

\textbf{Gram matrix distortion analysis.} The Gram matrices $G_T = \bar{E}_T \bar{E}_T^\top \in \mathbb{R}^{V \times V}$ and $G_S = \bar{E}_T WW^\top \bar{E}_T^\top \in \mathbb{R}^{V \times V}$ encode all pairwise inner products. The distortion matrix is $G_T - G_S = \bar{E}_T (I_{d_T} - WW^\top) \bar{E}_T^\top$. Computing the squared Frobenius norm:
\begin{equation}
\begin{aligned}
\|G_T - G_S\|_F^2 &= \mathrm{tr}((G_T - G_S)^2) \\
&= \mathrm{tr}(\bar{E}_T (I - WW^\top) \bar{E}_T^\top \bar{E}_T (I - WW^\top) \bar{E}_T^\top) \\
&= \mathrm{tr}( (I - WW^\top) (\bar{E}_T^\top \bar{E}_T)
      (I - WW^\top) (\bar{E}_T^\top \bar{E}_T) ) \\
&= V \cdot \mathrm{tr}((I - WW^\top) VC (I - WW^\top) C) \\
&= V^2 \cdot \mathrm{tr}((I - WW^\top) C^2 (I - WW^\top))
\end{aligned}
\end{equation}
using the cyclicity of trace.

\textbf{Spectral analysis.} Since $W = U[:, 1:d_S]$ and $U$ contains the eigenvectors of $C$, the complementary projector satisfies $I - WW^\top = \sum_{k=d_S+1}^{d_T} u_k u_k^\top$. Therefore,
\begin{equation}
(I - WW^\top) C = \sum_{k=d_S+1}^{d_T} \lambda_k u_k u_k^\top
\end{equation}
by orthonormality. Computing the product yields
\begin{equation}
(I - WW^\top) C (I - WW^\top) C = \sum_{k=d_S+1}^{d_T} \lambda_k^2 u_k u_k^\top
\end{equation}
with trace $\mathrm{tr}((I - WW^\top) C^2 (I - WW^\top)) = \sum_{k=d_S+1}^{d_T} \lambda_k^2$. Therefore, the Gram matrix distortion is
\begin{equation}
\|G_T - G_S\|_F^2 = V^2 \sum_{k=d_S+1}^{d_T} \lambda_k^2
\end{equation}

This result enables principled compression of embedding tables while preserving semantic relationships. The distortion depends only on the discarded eigenvalues, allowing practitioners to choose $d_S$ by setting an acceptable distortion threshold or retained variance ratio. The quadratic dependence on eigenvalues means that if the spectrum decays rapidly—as often occurs when categorical features exhibit strong correlation structure (e.g., $\lambda_k \sim k^{-\alpha}$ with $\alpha > 1$)—excellent compression is achievable with minimal distortion to the semantic relationships captured by inner products.
\end{proof}

\section{Offline Evaluation Implementation Details (Sec.~\ref{sec:exp_setup})}
\label{app:implementation_detail}

\paragraph{Dataset Descriptions.}
We use three widely adopted CTR prediction benchmarks.  
\textit{Criteo}\footnote{\href{https://www.kaggle.com/c/criteo-display-ad-challenge}{https://www.kaggle.com/c/criteo-display-ad-challenge}} contains 45 million display advertising samples collected over one week, with 26 categorical and 13 numerical feature fields.  
\textit{Avazu}\footnote{\href{https://www.kaggle.com/c/avazu-ctr-prediction}{https://www.kaggle.com/c/avazu-ctr-prediction}} consists of 10 days of mobile ad click logs with 22 categorical features, such as app category and device ID.  
\textit{Criteo1T}\footnote{\href{https://huggingface.co/datasets/criteo/CriteoClickLogs}{https://huggingface.co/datasets/criteo/CriteoClickLogs}} is a large-scale variant constructed from 24 days of Criteo data by sampling 3\% per day, forming 103 million samples. It captures long-term temporal dynamics and tests model scalability.

Table~\ref{tab:dataset_statistics} summarizes the key statistics and data split ratios of all datasets.

\begin{table}[t]
\centering
\caption{Dataset statistics. \#Samples denotes the total number of samples. CTR represents the positive rate (click-through rate). NF/CF indicate the number of numerical and categorical features, respectively. Data Split Ratio represents $\mathcal{D}^{\text{hist}}:\mathcal{D}^{\text{train}}:\mathcal{D}^{\text{online}}:\mathcal{D}^{\text{test}}$.}
\vspace{-1em}
\label{tab:dataset_statistics}
\small
\resizebox{\linewidth}{!}{
\begin{tabular}{l|ccccc}
\hline
\rowcolor{blue!15}
\textbf{Dataset} & \textbf{\#Samples} & \textbf{CTR (\%)} & \textbf{\#NF} & \textbf{\#CF} &\textbf{Data Split Ratio} \\
\hline
Criteo & 45,840,617 & 25.6 & 13 & 26 & 4:4:1:1\\
Avazu & 40,428,967 & 17.0 & 0 & 22 & 4:4:1:1\\
Criteo1T & 103,983,943 & 3.3 & 13 & 26 & 10:8:1:1\\
% \hline
% Industrial  & >1B & - & - & - & 5y:6m:5d:5d\\
\hline
\end{tabular}
}
\vspace{-1em}
\end{table}

For the public datasets, we follow the preprocessing pipeline in~\cite{song2019autoint}. Specifically, features with occurrence frequencies below predefined thresholds are replaced with a default ``<UNK>'' token, where the thresholds are set to 10 for Criteo, 5 for Avazu and 30 for Criteo1T, respectively. For numerical features in Criteo and Criteo1T, we apply log-transformation to mitigate large variance by transforming each value $x$ to $\log_2(x)$ if $x > 2$.

\paragraph{Model Configurations.}
To examine CrossAdapt under architectural heterogeneity, we design two teacher-student settings. The teacher model is a Multi-Layer Perceptron (MLP)~\cite{covington2016deep} with an embedding dimension of 8 and hidden layers of [64, 32, 4] for binary CTR prediction.  
The student models include an \textit{architecture-matched} setting using the same MLP as a homogeneous baseline, and an \textit{architecture-mismatched} setting adopting FiBiNET~\cite{huang2019fibinet} with an embedding dimension of 16, SENET and bilinear-interaction layers, and hidden layers of [256, 128, 16]. This heterogeneous setup introduces clear differences in embedding size, interaction mechanisms, and capacity, offering a strong test for cross-architecture transfer.

To further assess robustness, we also experiment with several representative CTR models as teachers or students, including MLP, AFM~\cite{xiao2017attentional}, DCN~\cite{wang2017deep}, AutoInt~\cite{song2019autoint}, and FiBiNET, covering diverse interaction paradigms and complexities.

\paragraph{Hyperparameter Settings.} 
For knowledge distillation, we set the temperature $\tau = 4.0$ and the loss balance coefficient $\lambda = 0.7$ to balance the distillation loss and task-specific loss. During offline sampling, 10\% of accessible historical data is selected with a class-balanced ratio ($r_\text{pos} = r_\text{neg} = 0.5$). In the online adaptation phase, the student model is updated at every step, while the teacher model is updated every 10 steps. For distribution shift detection, we employ 10 equal-sized consecutive windows with 50 histogram bins each, using Jensen-Shannon (JS) divergence as the metric with thresholds $\theta_\text{low} = 0.01$ and $\theta_\text{high} = 0.05$. Historical augmentation incorporates up to 10\% of samples from stable periods ($r_\text{enh} = 0.1$). All hyperparameters remain consistent across experiments unless explicitly stated otherwise.

We employ the Adam optimizer with learning rates of 0.1 and 0.001 for the embedding tables and interaction networks, respectively. The training batch size is set to 4096. Each experiment is repeated five times with different random seeds, and we report the mean and standard deviation of the results.

\paragraph{Experimental Environment.} 
All offline evaluation experiments are conducted on a workstation equipped with an NVIDIA RTX 4070Ti SUPER GPU (16GB VRAM), an Intel Core i5-12600KF CPU (10 cores @ 3.7 GHz), and Ubuntu 22.04 LTS. The software environment comprises PyTorch 2.6.0, Python 3.10, and CUDA 12.4.

\begin{figure}[t]
\centering
\begin{subfigure}[b]{\linewidth}
    \centering
    \includegraphics[width=0.9\linewidth]{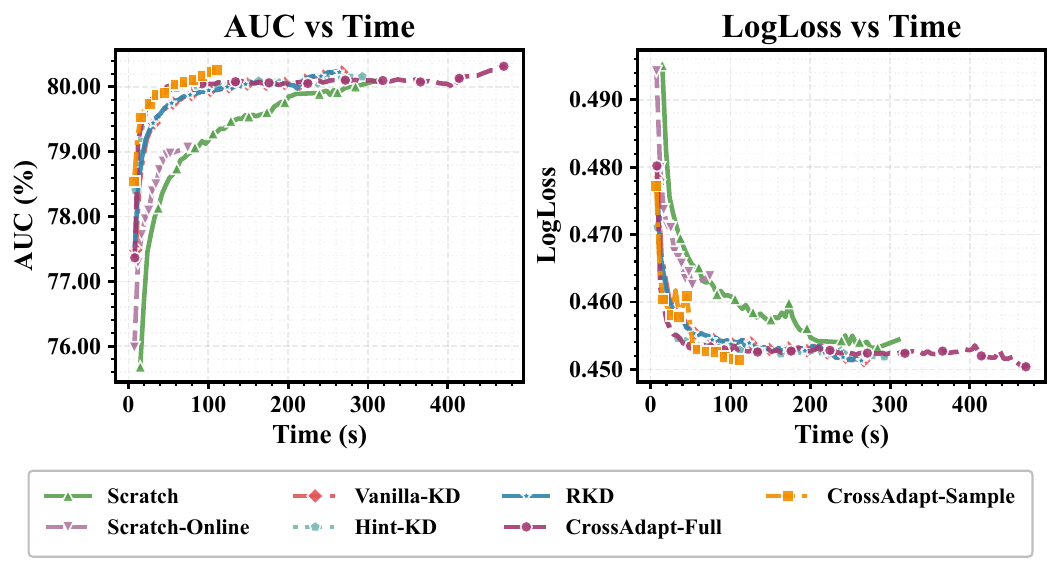}
    \caption{Criteo}
    \label{fig:time_criteo}
\end{subfigure}

\begin{subfigure}[b]{\linewidth}
    \centering
    \includegraphics[width=0.9\linewidth]{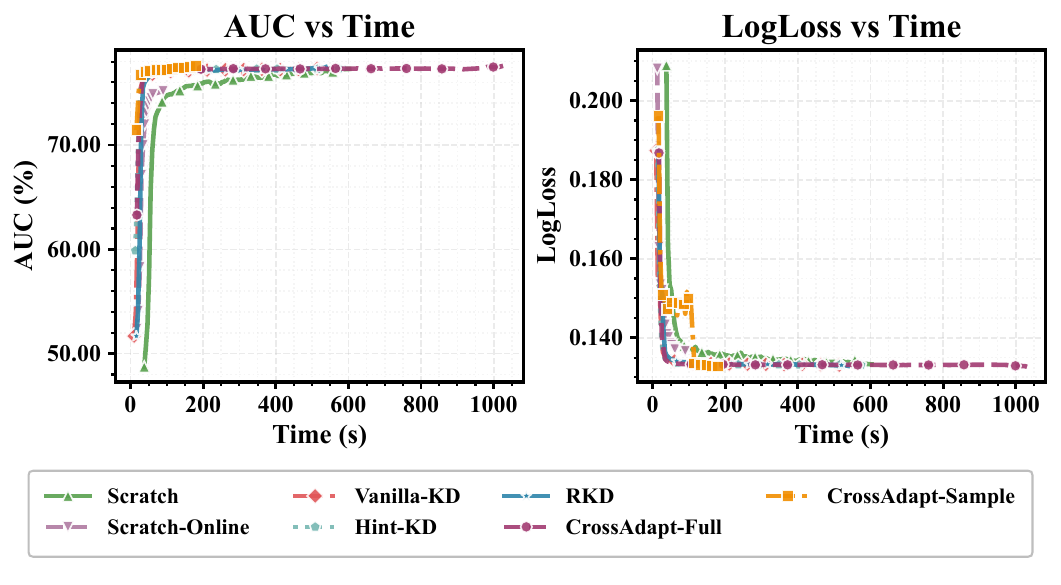}
    \caption{Criteo1T}
    \label{fig:time_criteo1T}
\end{subfigure}
\vspace{-2em}
\caption{Online training convergence curves on Avazu and Criteo1T. CrossAdapt achieves faster convergence compared to baselines.}
\label{fig:time_criteo2}
\vspace{-1em}
\end{figure}

\section{Training Efficiency and Convergence Analysis (Sec.~\ref{sec:rq1})}
\label{app:convergence_analysis}
Figure~\ref{fig:time_criteo} illustrates the convergence behavior on Criteo. CrossAdapt reaches stable performance ($\geq$80.2\% AUC) within 100 seconds, whereas Scratch requires over 200 seconds—representing a 50\% reduction in convergence time. Similar acceleration is observed on Criteo1T (Figure~\ref{fig:time_criteo1T}), where CrossAdapt converges in 150 seconds compared to Scratch’s 400 seconds (62\% reduction). Traditional KD methods (Vanilla-KD, Hint-KD, RKD) achieve intermediate speeds, outperforming Scratch but still lagging behind CrossAdapt. These results highlight the efficiency of CrossAdapt, whose strategic distillation sample selection module enables notably faster convergence.

\begin{table}[t]
\centering
\caption{Architectural robustness analysis on Criteo dataset. Comparison of teacher model performance, scratch training, and CrossAdapt-Full across 25 architecture transfer scenarios.}
\vspace{-1em}
\label{tab:architecture_generalization}
\small
\resizebox{\linewidth}{!}{
\begin{tabular}{lcccc}
\hline
\rowcolor{blue!15}
$\mathcal{M}_T \rightarrow \mathcal{M}_S$ & \textbf{Teacher} & \textbf{Scratch} & \textbf{CrossAdapt} &$\Delta$ \\
\hline
MLP $\rightarrow$ MLP & 80.08±0.02 & 79.90±0.03 & \textbf{80.29±0.01} & \textcolor{blue}{+0.39}\\
MLP $\rightarrow$ AFM & 80.08±0.02 & 79.55±0.04 & \textbf{80.21±0.02} & \textcolor{blue}{+0.66}\\
MLP $\rightarrow$ DCN & 80.08±0.02 & 79.92±0.02 & \textbf{80.31±0.01} & \textcolor{blue}{+0.39}\\
MLP $\rightarrow$ AutoInt & 80.08±0.02 & 80.08±0.01 & \textbf{80.33±0.01} & \textcolor{blue}{+0.25}\\
MLP $\rightarrow$ FiBiNET & 80.08±0.02 & 80.08±0.01 & \textbf{80.33±0.01} & \textcolor{blue}{+0.25}\\
\hline
AFM $\rightarrow$ MLP & 79.53±0.17 & 79.90±0.03 & \textbf{80.08±0.07} & \textcolor{blue}{+0.18}\\
AFM $\rightarrow$ AFM & 79.53±0.17 & 79.55±0.04 & \textbf{79.98±0.07} & \textcolor{blue}{+0.43}\\
AFM $\rightarrow$ DCN & 79.53±0.17 & 79.92±0.02 & \textbf{80.06±0.10} & \textcolor{blue}{+0.14}\\
AFM $\rightarrow$ AutoInt & 79.53±0.17 & 80.08±0.01 & \textbf{80.09±0.07} & \textcolor{blue}{+0.01}\\
AFM $\rightarrow$ FiBiNET & 79.53±0.17 & 80.08±0.01 & \textbf{80.10±0.07} & \textcolor{blue}{+0.02}\\
\hline
DCN $\rightarrow$ MLP & 80.01±0.00 & 79.90±0.03 & \textbf{80.29±0.02} & \textcolor{blue}{+0.39}\\
DCN $\rightarrow$ AFM & 80.01±0.00 & 79.55±0.04 & \textbf{80.21±0.02} & \textcolor{blue}{+0.66}\\
DCN $\rightarrow$ DCN & 80.01±0.00 & 79.92±0.02 & \textbf{80.26±0.03} & \textcolor{blue}{+0.34}\\
DCN $\rightarrow$ AutoInt & 80.01±0.00 & 80.08±0.01 & \textbf{80.31±0.01} & \textcolor{blue}{+0.23}\\
DCN $\rightarrow$ FiBiNET & 80.01±0.00 & 80.08±0.01 & \textbf{80.31±0.01} & \textcolor{blue}{+0.23}\\
\hline
AutoInt $\rightarrow$ MLP & 80.05±0.02 & 79.90±0.03 & \textbf{80.34±0.01} & \textcolor{blue}{+0.44}\\
AutoInt $\rightarrow$ AFM & 80.05±0.02 & 79.55±0.04 & \textbf{79.99±0.17} & \textcolor{blue}{+0.44}\\
AutoInt $\rightarrow$ DCN & 80.05±0.02 & 79.92±0.02 & \textbf{80.33±0.01} & \textcolor{blue}{+0.41}\\
AutoInt $\rightarrow$ AutoInt & 80.05±0.02 & 80.08±0.01 & \textbf{80.34±0.01} & \textcolor{blue}{+0.26}\\
AutoInt $\rightarrow$ FiBiNET & 80.05±0.02 & 80.08±0.01 & \textbf{80.38±0.01} & \textcolor{blue}{+0.30}\\
\hline
FiBiNET $\rightarrow$ MLP & 80.19±0.02 & 79.90±0.03 & \textbf{80.41±0.02} & \textcolor{blue}{+0.51}\\
FiBiNET $\rightarrow$ AFM & 80.19±0.02 & 79.55±0.04 & \textbf{80.30±0.02} & \textcolor{blue}{+0.75}\\
FiBiNET $\rightarrow$ DCN & 80.19±0.02 & 79.92±0.02 & \textbf{80.42±0.01} & \textcolor{blue}{+0.50}\\
FiBiNET $\rightarrow$ AutoInt & 80.19±0.02 & 80.08±0.01 & \textbf{80.43±0.01} & \textcolor{blue}{+0.35}\\
FiBiNET $\rightarrow$ FiBiNET & 80.19±0.02 & 80.08±0.01 & \textbf{80.45±0.01} & \textcolor{blue}{+0.37}\\
\hline
\end{tabular}
}
\vspace{-1em}
\end{table}

\section{Architectural Robustness (Sec.~\ref{sec:rq3})} 
\label{app:architecture_generalization}
Table~\ref{tab:architecture_generalization} shows results for 25 teacher-student combinations on Criteo. CrossAdapt consistently outperforms Scratch across all configurations, with +0.01–0.75\% AUC gains. The largest improvements occur when transferring to AFM students (+0.43–0.75\%), while FiBiNET— the strongest teacher—yields +0.35–0.75\% across students. Even AFM, the weakest teacher, provides meaningful gains (+0.01–0.43\%), demonstrating effective transfer despite a 0.66\% teacher capacity gap. Cross-architecture transfers match or exceed homogeneous transfers: for example, FiBiNET $\to$ MLP (+0.51\%) outperforms MLP $\to$ MLP (+0.39\%), and DCN $\leftrightarrow$ AutoInt shows effective bidirectional transfer. These results confirm that our proposed CrossAdapt effectively bridges heterogeneous architectures, from simple MLPs to attention-based models (AFM, AutoInt), cross networks (DCN), and bilinear interactions (FiBiNET).

\begin{figure}[tbp]
\centering

\begin{subfigure}[b]{\linewidth}
    \centering
    \includegraphics[width=\linewidth]{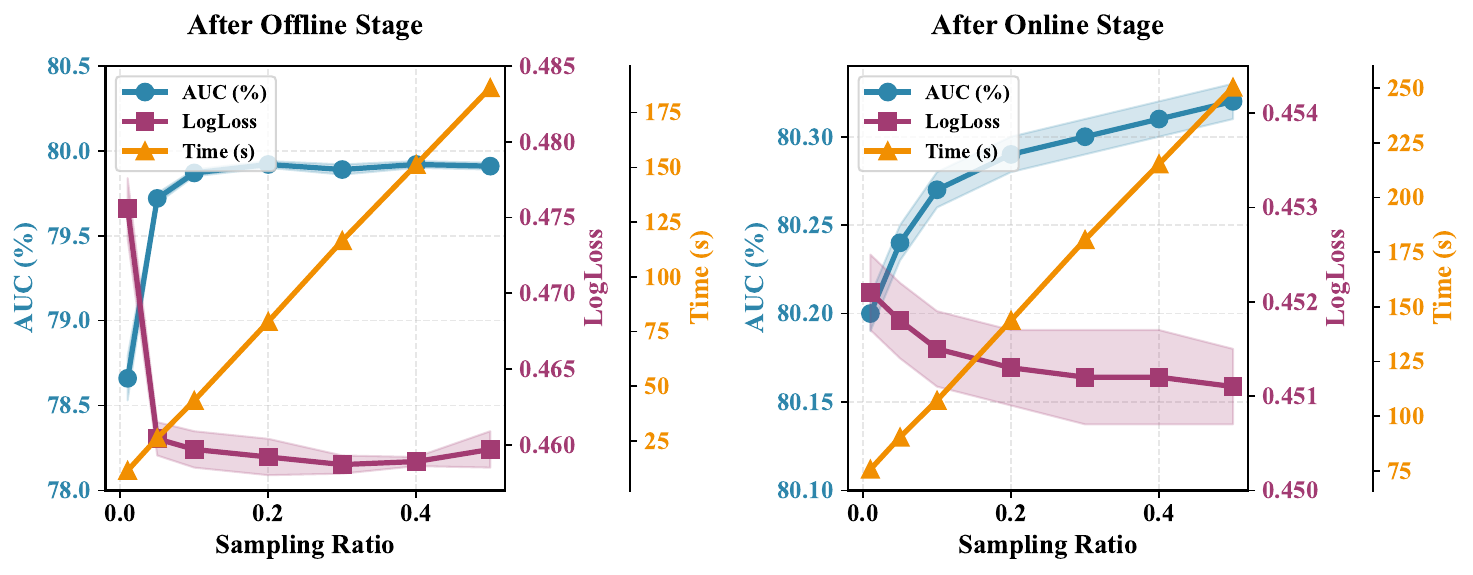}
    \caption{Sampling ratio $r$.}
    \label{fig:sampling_ratio}
\end{subfigure}

\begin{subfigure}[b]{\linewidth}
    \centering
    \includegraphics[width=\linewidth]{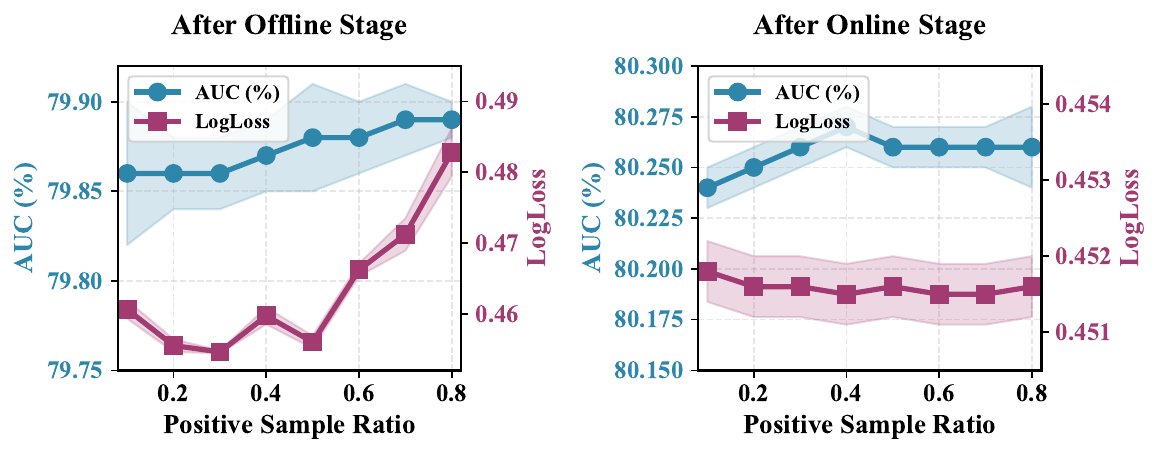}
    \caption{Positive sampling ratio $r_\text{pos}$.}
    \label{fig:positive_ratio}
\end{subfigure}

% \vspace{0.5em}

\begin{subfigure}[b]{0.495\linewidth}
    \centering
    \includegraphics[width=\linewidth]{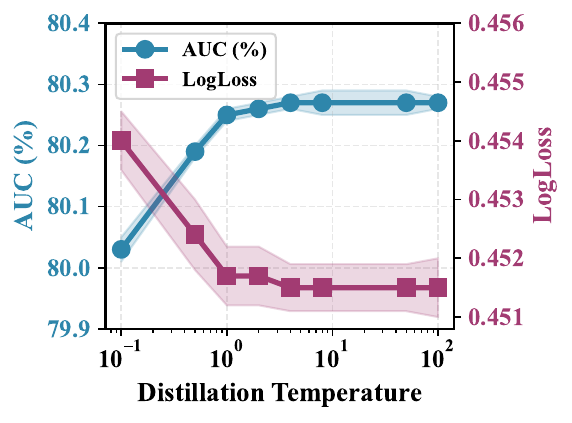}
    \caption{Distillation temperature $\tau$.}
    \label{fig:temperature}
\end{subfigure}
\hfill
\begin{subfigure}[b]{0.495\linewidth}
    \centering
    \includegraphics[width=\linewidth]{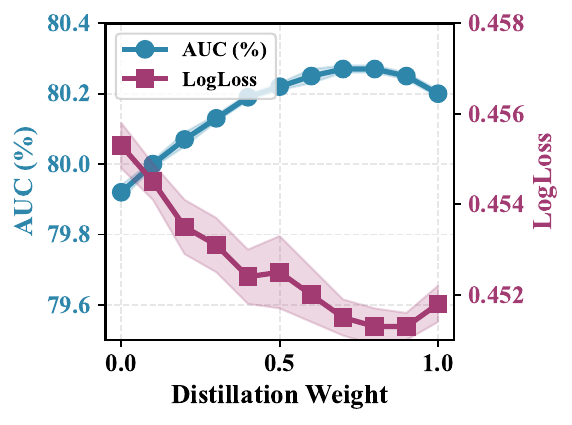}
    \caption{Distillation weight $\lambda$.}
    \label{fig:weight}
\end{subfigure}

\begin{subfigure}[b]{0.529\linewidth}
    \centering
    \includegraphics[width=\linewidth]{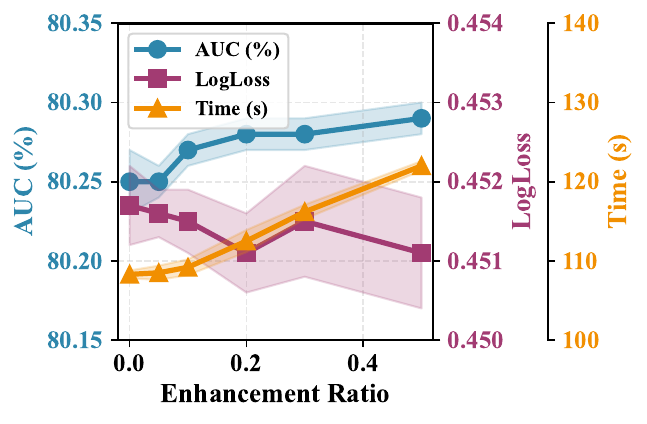}
    \caption{Enhancement ratio $r_\text{enh}$.}
    \label{fig:enhancement}
\end{subfigure}
\hfill
\begin{subfigure}[b]{0.462\linewidth}
    \centering
    \includegraphics[width=\linewidth]{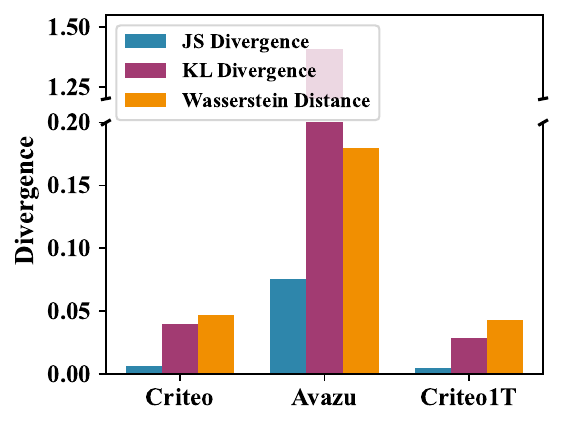}
    \caption{Distribution shift metrics $\mathbb{D}$.}
    \label{fig:divergence}
\end{subfigure}

\vspace{-1em}
\caption{Sensitivity analysis of key hyperparameters on Criteo dataset.}
\label{fig:hyperparameter_sensitivity}
\vspace{-1em}
\end{figure}

\section{Hyperparameter Analysis (Sec.~\ref{sec:rq3})}
\label{app:additional_hyperparameter}

We conduct comprehensive experiments to examine CrossAdapt's robustness across key hyperparameters. Figure~\ref{fig:hyperparameter_sensitivity} presents the performance variations with respect to the sampling ratio ($r$), positive sampling ratio ($r_\text{pos}$), distillation temperature ($\tau$), distillation balance coefficient ($\lambda$), historical augmentation ratio ($r_\text{enh}$), and distribution shift detection metrics. These results reveal both optimal operating points and the model's resilience to parameter perturbations across different experimental configurations.

\textbf{Impact of Sampling Ratio} (Figure~\ref{fig:sampling_ratio}). After the offline stage, AUC improves rapidly from 78.66\% ($r=0.01$) to 79.87\% ($r=0.1$), then plateaus at around 79.90\% for $r \geq 0.2$, while training time scales linearly from 11.4s to 186.4s. After online adaptation, the model achieves consistently strong performance with AUC ranging from 80.20\% to 80.32\% across all sampling ratios, demonstrating that CrossAdapt's online stage effectively compensates for lower offline sampling. 
We find that $r=0.1$ provides a good balance, achieving 80.27\% AUC with only 109.6s total time, while higher ratios yield only marginal gains ($\leq$0.05\% AUC) at substantially higher computational cost.

\textbf{Impact of Positive Sampling Ratio} (Figure~\ref{fig:positive_ratio}). 
During the offline stage, increasing the positive sampling ratio $r_\text{pos}$ from 0.1 to 0.8 yields a steady AUC improvement from 79.86\% to 79.89\%. However, LogLoss follows a non-monotonic trend—decreasing from 0.4606 ($r_\text{pos}=0.1$) to 0.4546 ($r_\text{pos}=0.3$) before rising to 0.4828 ($r_\text{pos}=0.8$)—indicating that excessively high positive ratios are suboptimal. After online adaptation, the influence of $r_\text{pos}$ diminishes, with all configurations converging to similar performance levels. The best results are obtained at $r_\text{pos}=0.4$, achieving the highest AUC (80.27\%) and the lowest LogLoss (0.4515). This suggests that an appropriate increase in positive sample ratio can boost distillation effectiveness, as positive samples carry more informative signals.

\textbf{Distillation Temperature Sensitivity} (Figure~\ref{fig:temperature}). 
The temperature parameter $\tau$ exhibits a clear plateau for $\tau \in [4, 100]$, consistently achieving the optimal AUC of 80.27\% and LogLoss of 0.4515. In contrast, lower temperatures degrade performance; for example, $\tau=0.1$ yields only 80.03\% AUC with a LogLoss of 0.4540. The stability within the optimal range $\tau \in [4, 100]$ highlights substantial operational flexibility, as the system remains insensitive to temperature selection in this interval.

\textbf{Distillation Balance Coefficient} (Figure~\ref{fig:weight}).
The distillation weight $\lambda$ exhibits a smooth performance curve peaking at $\lambda \in [0.7, 0.8]$, where both configurations achieve 80.27\% AUC with minimal LogLoss (0.4515 and 0.4513 respectively). Pure task loss ($\lambda=0$) yields substantially lower performance at 79.92\% AUC, demonstrating the critical importance of distillation knowledge. Performance steadily improves as $\lambda$ increases from 0.1 (80.00\% AUC) through the optimal range (80.27\% AUC), then slightly degrades at $\lambda=1$ (80.20\% AUC). The effective range $\lambda \in [0.6, 0.9]$ maintains AUC above 80.25\%, validating that balanced integration of task loss and distillation knowledge yields optimal results.

\textbf{Historical Augmentation Strategy} (Figure~\ref{fig:enhancement}).
Enhancement ratio $r_\text{enh}$ shows a monotonic improvement trend with increasing augmentation. Performance improves from 80.25\% AUC without augmentation to 80.29\% at $r_\text{enh}=0.5$, with corresponding LogLoss reduction from 0.4517 to 0.4511. The optimal configuration ($r_\text{enh}=0.5$) requires 122.0s training time, representing a 12.7\% increase over the baseline (108.3s). Notably, $r_\text{enh} \in [0.2, 0.5]$ achieves consistently strong performance (80.28-80.29\% AUC), suggesting that a moderate level of historical augmentation can enrich the training distribution and thereby enhance distillation effectiveness.

\textbf{Distribution Shift Detection} (Figure~\ref{fig:divergence}).
Although different divergence metrics show substantial scale differences—JS Divergence ranges from 0.0046 (Criteo1T) to 0.0753 (Avazu), while KL Divergence spans from 0.0286 to 1.4099—their ranking of datasets remains consistent. Both metrics identify Avazu as having the largest distribution shift, whereas Criteo and Criteo1T show smaller deviations. This scale-invariant consistency indicates that the adaptive mechanism can effectively operate under multiple criteria, as long as appropriate thresholds are calibrated, ensuring reliable activation of enhancement strategies regardless of the chosen metric.

\begin{table}[tbp]\footnotesize
\centering
\caption{Notation Summary}
\vspace{-1em}
\label{tab:notation}
% \resizebox{0.88\linewidth}{!}{
\begin{tabular}{cl}
\toprule
\textbf{Notation} & \textbf{Description} \\
\midrule
\multicolumn{2}{l}{\textit{Data and Features}} \\
$\mathcal{D}^{\text{hist}}$ & Historical data (inaccessible) \\
$\mathcal{D}^{\text{train}}$ & Available offline training data \\
$\mathcal{D}^{\text{online}}$ & Online streaming data \\
$\mathcal{D}^{\text{test}}$ & Test data for evaluation \\
$\mathcal{D}_{\text{pos}}, \mathcal{D}_{\text{neg}}$ & Positive and negative samples \\
$\mathbf{x}$ & Feature vector \\
$\mathbf{x}^{\text{cat}}, \mathbf{x}^{\text{num}}$ & Categorical and numerical features \\
$y \in \{0,1\}$ & Binary label (click/no-click) \\
$N$ & Number of samples \\
$V$ & Vocabulary size/cardinality \\
$B$ & Batch size \\
$\mathcal{B}_t$ & Batch at time step $t$ \\
\midrule
\multicolumn{2}{l}{\textit{Models and Parameters}} \\
$f_T, f_S$ & Teacher and student models \\
$\theta, \theta_T, \theta_S$ & Model parameters (general, teacher, student) \\
$\theta^{\text{net}}$ & Interaction network parameters \\
$E \in \mathbb{R}^{V \times d}$ & Embedding matrix \\
$E_T, E_S$ & Teacher and student embedding matrices \\
$d, d_T, d_S$ & Embedding dimensions \\
$p_T, p_S$ & Teacher and student predictions \\
\midrule
\multicolumn{2}{l}{\textit{Loss Functions and Costs}} \\
$\mathcal{L}$ & General loss function \\
$\mathcal{L}_{\text{BCE}}$ & Binary cross-entropy loss \\
$\mathcal{L}_{\text{KD}}$ & Knowledge distillation loss \\
$\mathcal{C}_{\text{switch}}$ & Total model switching cost \\
$\mathcal{C}_{\text{comp}}$ & Computational overhead \\
$\mathcal{C}_{\text{perf}}$ & Performance degradation cost \\
\midrule
\multicolumn{2}{l}{\textit{Projection and Transfer}} \\
$W$ & Projection matrix \\
$Q, U$ & Orthogonal matrices from QR/eigen decomposition \\
$C$ & Covariance matrix \\
$\Lambda$ & Eigenvalue diagonal matrix \\
$\mu$ & Mean vector of embeddings \\
$\bar{E}_T$ & Centered embedding matrix \\
$G_T, G_S$ & Gram matrices (teacher and student) \\
\midrule
\multicolumn{2}{l}{\textit{Training Hyperparameters}} \\
$\eta_S, \eta_T$ & Student and teacher learning rates \\
$\lambda$ & Knowledge distillation weight \\
$\tau$ & Teacher update interval \\
$t$ & Time step \\
\midrule
\multicolumn{2}{l}{\textit{Sampling Strategy}} \\
$\Psi(\mathcal{D}, n)$ & Function to sample $n$ samples from dataset $\mathcal{D}$ \\
$r$ & Overall sampling ratio \\
$r_{\text{pos}}$ & Target positive sampling ratio \\
$r_{\text{unclick}}$ &Unclicked-to-clicked ratio\\
$K$ & Number of temporal blocks \\
\midrule
\multicolumn{2}{l}{\textit{Distribution Shift Detection}} \\
$W_i$ & Time window $i$ \\
$n$ & Number of time windows \\
$P_j^{W_i}$ & Probability distribution of feature $j$ in window $W_i$ \\
$\mathbb{D}(\cdot, \cdot)$ & Distance metric between distributions \\
$\Delta_j^{(i)}$ & Distribution shift for feature $j$ between windows $i$ and $i+1$ \\
$\Delta_\text{shift}$ & Overall distribution shift metric \\
$b$ & Number of histogram bins \\
\midrule
\multicolumn{2}{l}{\textit{Adaptive Enhancement}} \\
$r_\text{enh}$ & Historical data enhancement ratio \\
$\theta_\text{low}, \theta_\text{high}$ & Lower and upper shift thresholds \\
$k$ & Maximum enhancement ratio \\
$\mathcal{B}_\text{aug}$ & Augmented batch \\
\bottomrule
\end{tabular}
% }
\end{table}

\begin{algorithm}[tbp]\footnotesize
\caption{CrossAdapt: Offline Transfer and Online Adaptive Co-Distillation}
\label{alg:crossadapt}
\begin{algorithmic}[1]
\Require Teacher model $f_T(\theta_T)$ with embeddings $E_T$, untrained student model $f_S$, offline data $\mathcal{D}^{\text{train}}$, online stream $\mathcal{D}^{\text{online}}$, sampling ratio $r$, positive ratio $r_{\text{pos}}$, temporal blocks $K$, learning rates $\eta_S, \eta_T$, teacher update interval $\tau$, distillation weight $\lambda$, shift thresholds $\theta_{\text{low}}, \theta_{\text{high}}$
\Ensure Trained student model $f_S(\theta_S)$

\State \textbf{Stage 1: Offline Cross-Architecture Transfer}

\State // Step 1.1: Dimension-Adaptive Embedding Projection
\If{$d_S = d_T$}\Comment{Direct copy}
    \State $E_S \leftarrow E_T$ 
\ElsIf{$d_S > d_T$}\Comment{Expand to higher dimension}
    \State $W \leftarrow \text{OrthogonalProjection}(d_T, d_S)$, $E_S \leftarrow E_T \cdot W$ 
\Else \Comment{Reduce to lower dimension}
    \State $W \leftarrow \text{PCAProjection}(E_T, d_S)$, $E_S \leftarrow E_T \cdot W$ 
\EndIf

\State // Step 1.2: Strategic Distillation Sample Selection
\State Partition $\mathcal{D}^{\text{train}}$ into positive $\mathcal{D}_{\text{pos}}$ and negative $\mathcal{D}_{\text{neg}}$ samples;
Divide both into $K$ temporal blocks by time;
$\mathcal{D}'_{\text{train}} \leftarrow \emptyset$
\For{$k = 1$ to $K$}
    \State Compute sample sizes: $n_{\text{pos}} = \frac{r \cdot r_{\text{pos}} |\mathcal{D}_{\text{train}}|}{K}$, $n_{\text{neg}} = \frac{r(1-r_{\text{pos}})|\mathcal{D}_{\text{train}}|}{K}$
    \State $\mathcal{D}'_{\text{train}} \leftarrow \mathcal{D}'_{\text{train}} \cup \Psi(\mathcal{D}_{\text{pos}}^{(k)}, n_{\text{pos}}) \cup \Psi(\mathcal{D}_{\text{neg}}^{(k)}, n_{\text{neg}})$
\EndFor

\State // Step 1.3: Progressive Interaction Network Distillation
\State // \textbf{Phase 1:} Freeze embeddings $E_S$, train interaction network $\theta_S^{\text{net}}$
\For{each batch $\mathcal{B} \in \mathcal{D}'_{\text{train}}$}
    \State Get predictions: $p_T \leftarrow f_T(\mathcal{B})$, $p_S \leftarrow f_S(\mathcal{B})$
    \State $\mathcal{L}_{\text{net}} \leftarrow \mathcal{L}_{\text{BCE}}(p_S, y) + \lambda \mathcal{L}_{\text{KD}}(p_S, p_T)$, 
    $\theta_S^{\text{net}} \leftarrow \theta_S^{\text{net}} - \eta \nabla_{\theta_S^{\text{net}}} \mathcal{L}_{\text{net}}$
\EndFor
\State // \textbf{Phase 2:} Unfreeze $E_S$, jointly optimize all parameters $\theta_S$
\For{each batch $\mathcal{B} \in \mathcal{D}'_{\text{train}}$}
    \State Get predictions: $p_T \leftarrow f_T(\mathcal{B})$, $p_S \leftarrow f_S(\mathcal{B})$
    \State $\mathcal{L}_{\text{joint}} \leftarrow \mathcal{L}_{\text{BCE}}(p_S, y) + \lambda \mathcal{L}_{\text{KD}}(p_S, p_T)$, $\theta_S \leftarrow \theta_S - \eta \nabla_{\theta_S} \mathcal{L}_{\text{joint}}$
\EndFor

\State \textbf{Stage 2: Online Adaptive Co-Distillation}

\State // Step 2.1: Distribution Shift Detection
\State Partition $\mathcal{D}^{\text{train}}$ into $n$ consecutive time windows $\{W_1, \ldots, W_n\}$
\For{$i = 1$ to $n-1$}
    \For{each feature $j = 1$ to $V$}
        \State Compute distance $\Delta_j^{(i)} \leftarrow \mathbb{D}(P_j^{W_i}, P_j^{W_{i+1}})$
\EndFor
\EndFor
\State $\Delta_{\text{shift}} \leftarrow \frac{1}{n-1} \sum_{i=1}^{n-1} \frac{1}{V} \sum_{j=1}^{V} \Delta_j^{(i)}$, compute $r_{\text{aug}}$
\State // Step 2.2: Asymmetric Teacher-Student Co-Evolution
\State Initialize time step $t \leftarrow 0$, teacher gradient $g_T \leftarrow 0$
\For{each streaming batch $\mathcal{B}_t \in \mathcal{D}^{\text{online}}$}
    \State $t \leftarrow t + 1$
    \State $\mathcal{B}_{\text{aug}} \leftarrow \mathcal{B}_t \cup \Psi(\mathcal{D}^{\text{train}}, r_{\text{aug}} \cdot |\mathcal{B}_t|)$ \Comment{Augment with history}
    \State Get predictions: $p_T \leftarrow f_T(\mathcal{B}_{\text{aug}})$, $p_S \leftarrow f_S(\mathcal{B}_{\text{aug}})$
    \State $\mathcal{L}_S \leftarrow \mathcal{L}_{\text{BCE}}(p_S, y) + \lambda \mathcal{L}_{\text{KD}}(p_S, p_T)$
    \State $\theta_S^{t+1} \leftarrow \theta_S^t - \eta_S \nabla_{\theta_S} \mathcal{L}_S$ \Comment{Student update}
    \State $\mathcal{L}_T \leftarrow \mathcal{L}_{\text{BCE}}(p_T, y)$, $g_T \leftarrow g_T + \nabla_{\theta_T} \mathcal{L}_T$ 
    \If{$t \bmod \tau = 0$}\Comment{Periodic teacher update}
        \State $\theta_T^{t+1} \leftarrow \theta_T^t - \eta_T \cdot g_T$, $g_T \leftarrow 0$
    \Else
        \State $\theta_T^{t+1} \leftarrow \theta_T^t$
    \EndIf
\EndFor

\Return Trained student model $\theta_S$
\end{algorithmic}
\end{algorithm}

\section{Online Deployment Details (Sec.~\ref{sec:exp_online_setup})}
\label{app:online_deploy_detail}

\paragraph{Evaluation Metrics.}

We employ the following metrics to evaluate pCVR prediction quality:

\textbf{pCVR Bias} quantifies systematic deviation between predicted and actual conversion rates: $\text{Bias} = \frac{1}{|\mathcal{I}|} \sum_{i \in \mathcal{I}} | \frac{\hat{p}_i - p_i}{p_i} |$, where $\hat{p}_i$ and $p_i$ denote predicted and observed conversion probabilities for item $i$. Lower bias indicates better probability calibration, critical for bid optimization.

\textbf{Spearman Correlation} ($\rho$) measures rank-order agreement between predictions and ground truth: $\rho = 1 - \frac{6\sum d_i^2}{n(n^2-1)}$, where $d_i$ is the rank difference for item $i$. Values range from -1 to 1, with higher values indicating superior ranking quality.

\textbf{NDCG@K} assesses top-$K$ ranking quality with position-based discounting: $\text{NDCG@K} = \frac{\sum_{i=1}^{K} \frac{2^{\text{rel}_i} - 1}{\log_2(i+1)}}{\text{IDCG@K}}$, where $\text{rel}_i$ is the relevance at position $i$ and IDCG@K normalizes to [0,1]. We report NDCG@5 and NDCG@10 to emphasize top-position accuracy, crucial for advertising systems where high-ranked items dominate exposure.

For USAug evaluation (Table~\ref{tab:usaug_results}), we measure student-teacher prediction consistency on exposed items, where the production teacher serves as ground truth validated through extensive online testing.

\paragraph{Experimental Environment.}
Experiments are conducted on the Tencent AngelRec platform, equipped with 8 NVIDIA H20 GPUs (96 GB memory each), an AMD 96-core processor, and a parameter server with 5TB of memory.

\paragraph{Remark.}
The adaptive historical knowledge preservation module, which dynamically mixes historical data from different time periods during online training, currently faces infrastructure limitations in our production system and cannot be deployed at this stage. We validate its effectiveness on public datasets (Sec.~\ref{sec:ablation}) and plan to enable this capability as platform infrastructure evolves.

\section{Notations and Algorithms}
Table~\ref{tab:notation} summarizes all the notations used throughout this paper.
Algorithm~\ref{alg:crossadapt} presents the full procedure of the proposed CrossAdapt algorithm.

\end{document}